\documentclass[10pt,twocolumn,letterpaper]{article}

\usepackage{cvpr}              %
\usepackage{amsmath}
\usepackage{algorithm}
\usepackage{algorithmicx}
\usepackage{algpseudocode}
\usepackage{amsthm}
\usepackage{colortbl,booktabs}
\usepackage{ragged2e}
\usepackage{makecell}
\usepackage{tabularx}
\usepackage{textcomp}
\usepackage{tabu} 
\usepackage[shortlabels]{enumitem}
\usepackage{multirow}
\usepackage{wrapfig}
\usepackage{pifont}       
\usepackage{bbding}
\usepackage{arydshln}
\usepackage{fontawesome}
\usepackage{footnote}
\usepackage{tablefootnote}
\theoremstyle{definition}
\newtheorem{assumption}{Assumption} %
\newtheorem{proposition}{Proposition}
\newtheorem{lemma}{Lemma}
\definecolor{c6}{HTML}{95baa6}
\definecolor{c7}{HTML}{76a2bb}

\newcommand{\bl}[1]{\textcolor{c7}{#1}}

\definecolor{c5}{HTML}{b71a3b}
\newcommand{\myred}[1]{\textcolor{c5}{#1}}

\definecolor{mycolor}{HTML}{a7caea} 
 
\definecolor{mygreen}{HTML}{5B9479}
\definecolor{myblue}{HTML}{4b8dbc}
\definecolor{myyellow}{HTML}{E2CD89}
\usepackage{soul}
\newcommand{\mygreen}[1]{\textcolor{mygreen}{#1}}
\newcommand{\myyellow}[1]{\textcolor{myyellow}{#1}}

\definecolor{cvprblue}{rgb}{0.21,0.49,0.74}
\usepackage[pagebackref,breaklinks,colorlinks,allcolors=cvprblue]{hyperref}

\def\paperID{} %
\def\confName{CVPR}
\def\confYear{2026}

\title{
\textit{\bl{Mo}\mygreen{D}\myred{ES}}: Accelerating Mixture-of-Experts \bl{M}ultim\bl{o}dal Large Language Models via \mygreen{D}ynamic \myred{E}xpert \myred{S}kipping
\vspace{-0.2in}
}

\author{
Yushi Huang\textsuperscript{1}, Zining Wang\textsuperscript{2}, Zhihang Yuan\textsuperscript{3}\thanks{Correspondence to: Zhihang Yuan (\texttt{hahnyuan@gmail.com}), Jun Zhang (\texttt{eejzhang@ust.hk}).}, Ruihao Gong\textsuperscript{2}, Yifu Ding\textsuperscript{2}, \\
Jinyang Guo\textsuperscript{2}, Xianglong Liu\textsuperscript{2}, Jun Zhang\textsuperscript{1}\footnote[1]{}\\
{\small \textsuperscript{1}Hong Kong University of Science and Technology \quad \textsuperscript{2}Beihang University \quad \textsuperscript{3}Peking University}\\
\vspace{-0.3in}
}

\begin{document}
\maketitle
\begin{abstract}
Mixture-of-Experts (MoE) Multimodal large language models (MLLMs) excel at vision–language tasks, but they suffer from high computational inefficiency. To reduce inference overhead, expert skipping methods have been proposed to deactivate redundant experts based on the current input tokens. However, we find that applying these methods—originally designed for unimodal large language models (LLMs)—to MLLMs results in considerable performance degradation. This is primarily because such methods fail to account for the heterogeneous contributions of experts across MoE layers and modality-specific behaviors of tokens within these layers. Motivated by these findings, we propose MoDES, the first training-free framework that adaptively skips experts to enable efficient and accurate MoE MLLM inference. It incorporates a globally-modulated local gating (GMLG) mechanism that integrates global layer-wise importance into local routing probabilities to accurately estimate per-token expert importance. A dual-modality thresholding (DMT) method is then applied, which processes tokens from each modality separately, to derive the skipping schedule. To set the optimal thresholds, we introduce a frontier search algorithm that exploits monotonicity properties, cutting convergence time from several days to a few hours. Extensive experiments for 3 model series across 13 benchmarks demonstrate that MoDES far outperforms previous approaches. For instance, when skipping 88\% experts for Qwen3-VL-MoE-30B-A3B-Instruct, the performance boost is up to \textbf{10.67\%} (97.33\% vs. 86.66\%). Furthermore, MoDES significantly enhances inference speed, improving the prefilling time by \textbf{2.16$\times$} and the decoding time by \textbf{1.26$\times$}. Our code is available at \url{https://github.com/ModelTC/MoDES}.
\end{abstract}
    
\section{Introduction}
\label{sec:intro}

\begin{figure}[!ht]
   \centering
    \setlength{\abovecaptionskip}{0.2cm}
   \begin{minipage}[b]{\linewidth}
       \begin{minipage}[b]{0.49\linewidth}
            \centering
            \begin{subfigure}[tp!]{\textwidth}
            \centering
            \includegraphics[width=\linewidth]{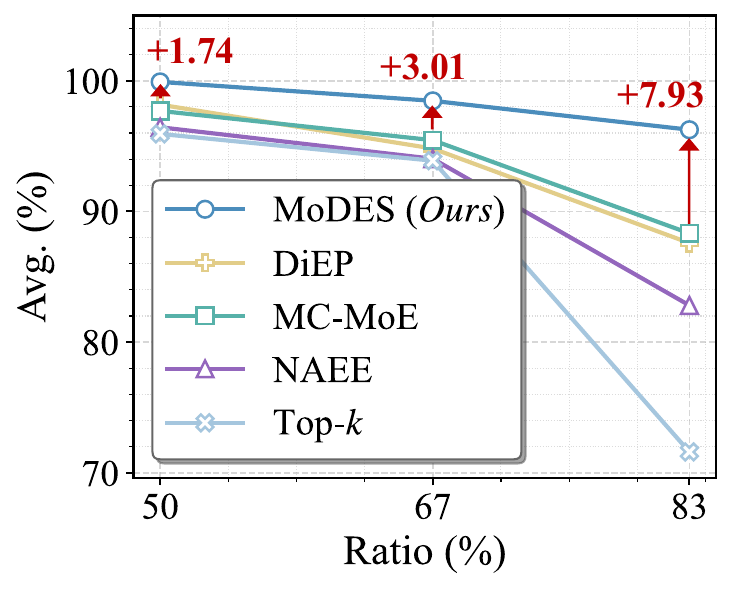}
            \end{subfigure}
       \end{minipage}
       \begin{minipage}[b]{0.49\linewidth}
            \centering
            \begin{subfigure}[tp!]{\linewidth}
            \centering
            \includegraphics[width=\linewidth]{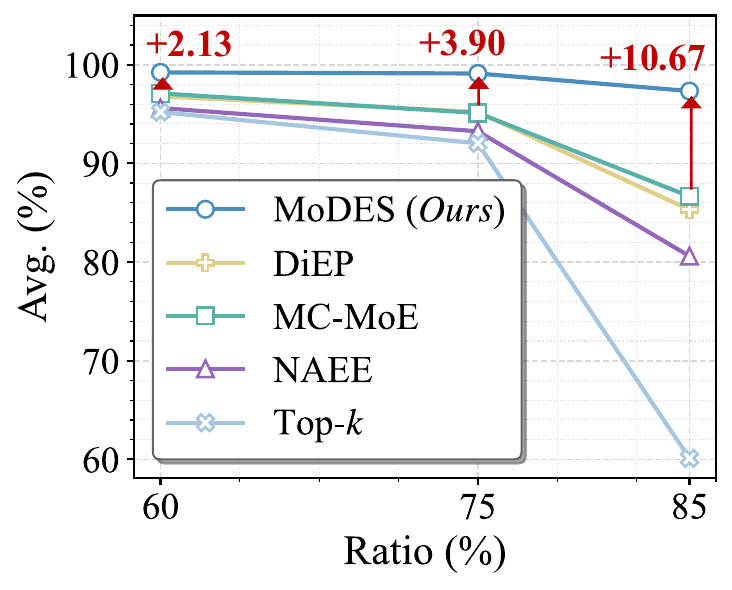}
            \end{subfigure}
       \end{minipage}
   \end{minipage}
   \vspace{-0.1in}
    \caption{Average performance (\%) \emph{vs.} \textit{expert skipping} ratios (\%) across different models~\cite{wang2025internvl3, qwen3_vl_moe_doc, team2025kimi} and methods~\cite{bai2025diep, huang2024mixture, lu2024not} on 13 benchmarks (as detailed in Sec.~\ref{sec:imple}). The \textit{left} subfigure is for Kimi-VL-A3B-Instruct~\cite{team2025kimi} and the \textit{right} subfigure  is for Qwen3-VL-MoE-30B-A3B-Instruct~\cite{qwen3_vl_moe_doc}.}
    \label{fig:teaser}
    \vspace{-0.2in}
\end{figure}
Multimodal large language models (MLLMs)~\citep{radford2021learning,vteam2025glm45vglm41vthinkingversatilemultimodal} have become a dominant paradigm for vision-language understanding tasks, showing remarkable performance in integrating text, images, and videos. However, as the scale of models keeps increasing to handle richer data and more complex tasks, they face significant computational bottlenecks during inference. For instance, Qwen2-VL~\cite{wang2024qwen2} with 72B parameters only achieves $<$10 tokens/s when processing a 4K-token input on 2$\times$A100 GPUs. This is because each token requires computations with all model parameters. The mixture-of-experts (MoE)~\cite{shazeer2017outrageouslylargeneuralnetworks} architecture has emerged as an effective solution to reduce the cost of large-scale MLLMs. By sparsely activating partial parameters (\emph{i.e.}, selected expert networks) for each token, MoE MLLMs~\cite{team2025kimi, qwen3_vl_moe_doc} decouple the factor of model size from computational costs. This design offers substantial computational savings without compromising performance~\cite{jiang2024mixtralexperts,lin2024moe}.

Nevertheless, MoE models typically struggle with suboptimal expert utilization~\cite{lu2024not,jin2024moeacceleratingmixtureofexpertsmethods} due to a fixed number of activated experts for all tokens, which can incur significant inference inefficiency~\cite{jin2024moeacceleratingmixtureofexpertsmethods,yue2024ada,lu2024not}. Recent \textit{expert skipping} methods~\citep{huang2024mixture, lu2024not, bai2025diep} thus propose to skip redundant experts \emph{w.r.t.} current tokens to accelerate inference. However, applying these methods to MoE MLLMs leads to a significant drop in accuracy. For example, as shown in Fig.~\ref{fig:teaser}, skipping 83\% of the experts in previous methods~\cite{lu2024not,bai2025diep,huang2024mixture} during inference results in accuracy drops of over 10\%.

To solve the problem, we first make in-depth analyses and obtain two key insights overlooked before: (\textit{i}) The contributions of experts to the model outputs vary significantly across layers. Specifically, experts in shallow layers play far more critical roles than those in deeper layers. However, prior works~\citep{huang2024mixture, lu2024not, bai2025diep} only consider \textit{intra-layer} information (\emph{e.g.}, Eq.~(\ref{eq:logits})) to develop skipping schedules. (\textit{ii}) Tokens of different modalities (\emph{i.e.}, text and vision) exhibit distinct behaviors as they pass through experts, and experts have a larger effect on updating text tokens. Yet prior works mainly study unimodal LLMs~\cite{kim2023mixture} and do not account for this modality gap in MLLMs. These observations underscore the need for a modality-specific \textit{expert skipping} method that explicitly models layer-specific contributions.

To this end, we introduce \textit{MoDES (Multimodal Dynamic Expert Skipping)}, the first accurate and efficient \textit{expert skipping} framework tailored for MoE MLLMs. In response to the first insight, we propose a \textit{globally-modulated local gating (GMLG)} mechanism, which combines global layer-specific importance with local routing probabilities to construct expert importance scores. The global importance is obtained via offline calibration with no inference-time overhead. Then, we introduce a \textit{dual-modality thresholding (DMT)} method which skips redundant experts whose importance scores for the current token fall below the threshold corresponding to the token’s modality. This modality-specific treatment considerably enhances the performance of \textit{expert skipping} for MLLMs. To determine the optimal thresholds, we further propose a \textit{frontier search} algorithm on a given search space. This search method leverages monotonicity properties of the performance loss and efficiency \emph{w.r.t.} thresholds, reducing the search time from more than $2$ days to less than $2$ hours for models with tens of billions of parameters without compromising performance.

To demonstrate the effectiveness of our method, we conduct extensive experiments on 3 MLLM families across 13 image and video understanding benchmarks. As shown in Fig.~\ref{fig:teaser}, the results indicate that MoDES consistently surpasses state-of-the-art (SOTA) methods. Notably, with extremely high expert skipping ratios ($>$80\%), MoDES achieves \textbf{7.93-10.67\%} performance enhancements compared with baselines while retaining $>$95\% accuracy of original models. Moreover, our MoDES yields a significantly \textbf{2.03$\times$} speedup in prefilling and a \textbf{1.24$\times$} speedup in decoding for Qwen3-VL-MoE-30B-A3B-Instruct~\cite{qwen3_vl_moe_doc}.

\section{Related Work}
\label{sec:related_work}

\noindent\textbf{Multimodal large language models.} Multimodal Large language models (MLLMs)~\cite{li2024llava,liu2023llava,bai2023qwenvl}, which build upon the success of large language models (LLMs)~\citep{achiam2023gpt, chiang2023vicuna, li2024llama, bai2023qwen}, have become a dominant paradigm for vision-language tasks~\citep{chen2024internvl, huang2024ffaa, wang2024qwen2, li2024mini, ataallah2024minigpt4, li2024llava}. However, as MLLMs~\cite{liu2024llava1.5, li2024llava-ov, hu2024mqt-llava} advance to handle higher resolutions and more video frames, the escalating number of visual tokens creates a severe computational bottleneck. Current advanced MLLMs~\citep{wang2025internvl3,team2025kimi, lin2024moe, qwen3_vl_moe_doc} adopt the mixture-of-experts (MoE)~\cite{fedus2022switch} architecture to reduce computational costs by processing each token with a subset of expert networks. Despite this, computation between tokens and multiple activated experts still incurs substantial overhead~\cite{liu2025netmoe, Dhasade_2025}.

\noindent\textbf{Efficient MoE.} Existing works on efficient MoE models can be categorized into training-aware and training-free approaches. Training-aware methods enhance routing balance and expert utilization during training~\citep{yue2024ada, guo2024dynamic, bai2025diep, wu2025routingexpertslearningroute}, but they necessitate costly retraining and extensive data access. In contrast, training-free techniques~\cite{wu2025acceleratingmultimodallargelanguage,gong2024llmcbenchmarkinglargelanguage,lv2025llmcbenchmarkingvisionlanguagemodel} enable lightweight efficiency enhancement without modifying the training pipeline, including quantization for parameter compression~\citep{duanmu2025mxmoe, kim2023mixture} and pruning for structural sparsity~\citep{lee2024stun, xie2024moe}. Owing to the modular and sparse nature of MoE, a new line of research—\textit{expert skipping}—has emerged, which dynamically bypasses redundant experts~\citep{bai2025diep, lu2024not, huang2024mixture} to speed up inference. Among these studies, Lu \emph{et al.}~\citep{lu2024not} utilize dynamic \textit{expert skipping} based on expert routing probabilities. MC-MoE~\citep{huang2024mixture} further integrates an attention-aware expert protection approach during skipping and combines mixed-precision quantization for expert compression. Additionally, DiEP~\citep{bai2025diep} introduces a differentiable expert pruning framework with adaptive \textit{expert skipping}, which jointly considers routing probabilities and expert similarity. However, these skipping methods are primarily developed for text-only LLMs~\cite{jiang2024mixtralexperts}, which limits their scalability to complex multimodal architectures. In contrast, our training-free \textit{expert skipping} framework focuses on advanced MoE MLLMs, achieving efficient inference without sacrificing cross-modal understanding.

\section{Preliminaries}\label{sec:preliminaries}
\noindent\textbf{Architecture of MLLM.} A typical MLLM~\citep{wang2024qwen2vl,bai2025qwen25vl,chen2024internvl2} comprises three core components: A visual encoder, a projector, and an LLM backbone. The visual encoder first extracts visual tokens from an image or video. The projector then aligns these tokens with the LLM's text embedding space. Finally, the LLM backbone, a stack of transformer layers~\cite{vaswani2017transformer} composed of self-attention and feed-forward networks (FFNs), processes the combined visual and text tokens to generate responses.

\noindent\textbf{Mixture-of-Experts (MoE).} %
The advanced MLLMs~\cite{team2025kimi, zhu2025internvl3} employ Mixture-of-Experts (MoE)~\cite{huang2024mixture} layers as their FFNs of the LLM backbones. This structure can be viewed as a conditional computation module composed of multiple parallel experts. Formally, let the $l$-th MoE layer contain $M$ experts, \emph{i.e.}, $\{\texttt{Expert}^{(l)}_1, \dots, \texttt{Expert}^{(l)}_M\}$, each of which is implemented as a multi-layer perception (MLP). Given an input token representation $\mathbf{x}^{(l)} \in \mathbb{R}^{d}$ ($d$ denotes hidden dimension), a lightweight router predicts a set of routing logits $\mathbf{r}^{(l)} = \{r^{(l)}_1,\dots,r^{(l)}_M\}$. These logits are then normalized into routing probabilities through a softmax operation:
\begin{equation}
\label{eq:logits}
    \pi_m^{(l)} = \frac{\exp(r_m^{(l)})}{\sum^{M}_{\hat{m}=1}\exp(r_{\hat{m}}^{(l)})},
\end{equation}
where $\pi^{(l)}_m$ reflects the contribution of $\texttt{Expert}^{(l)}_m$. To ensure sparse activation, only a subset of experts is executed. Let $\mathcal{S}^{(l)}$ denote the indices of the top-$k$ experts with the largest routing probabilities. The output $\mathbf{y}^{(l)}$ of the MoE layer is obtained through a weighted aggregation:
\begin{equation}
    \label{eq:aggregation}
\mathbf{y}^{(l+1)}=\sum_{m\in\mathcal{S}^{(l)}}\pi_m^{(l)}\cdot\texttt{Expert}^{(l)}_m(\mathbf{x}^{(l)}).
\end{equation}
This formulation allows the model to scale the number of parameters independently of the active computation cost.

\section{Motivation}\label{sec:motivation}
Existing studies~\cite{lu2024not,huang2024mixture,bai2025diep} have found that not every selected expert provides essential contributions for tokens. They thus propose to skip the computation of unimportant experts to improve inference efficiency. However, they focus on text-only LLMs~\cite{jiang2024mixtralexperts}. In this study, we have identified that directly adapting these methods to MoE MLLMs~\cite{team2025kimi, qwen3_vl_moe_doc} overlooks two key factors: Global contribution (Sec.~\ref{sec:motivation_global}) and modality gap (Sec.~\ref{sec:motivation_modality}). Both factors significantly affect the performance and efficiency of \textit{expert skipping} in MLLMs.

\subsection{Global Contribution Disregard}\label{sec:motivation_global}
Recent skipping strategies~\cite{huang2024mixture, lu2024not,bai2025diep} rely on the \emph{local} routing probabilities (Eq.~(\ref{eq:logits})) to determine the skipping schedule of the $l$-th layer, reflecting only input-dependent gating within a single layer. Such layer-agnostic rules ignore the \emph{global} contribution (\emph{i.e.}, impact on final outputs) imbalance of experts across different layers. Empirically, as shown in Fig.~\ref{fig:motivation_global}, we observe that when reducing the value of $k$ for expert routing, shallower layers incur much severe performance drops than those of deeper layers. This may result from that, relative to the error of deeper layers, errors introduced in shallow layers are amplified by subsequent layers~\cite{huang2025determininglayerwisesparsitylarge}, leading to a significant \textit{error explosion}.  Accordingly, the aforementioned layer-independent \textit{expert skipping} strategies~\cite{huang2024mixture, lu2024not, bai2025diep} risk excessive skipping at shallow layers, which are critical to final outputs, and \emph{vice versa} for deep layers. 
\begin{figure}[!ht]
\vspace{-0.1in}
   \centering
    \setlength{\abovecaptionskip}{0.2cm}
   \begin{minipage}[b]{\linewidth}
       \begin{minipage}[b]{0.326\linewidth}
            \centering
            \begin{subfigure}[tp!]{\textwidth}
            \centering
            \subcaption{ChartQA~\cite{masry2022chartqa}}
            \includegraphics[width=\linewidth]{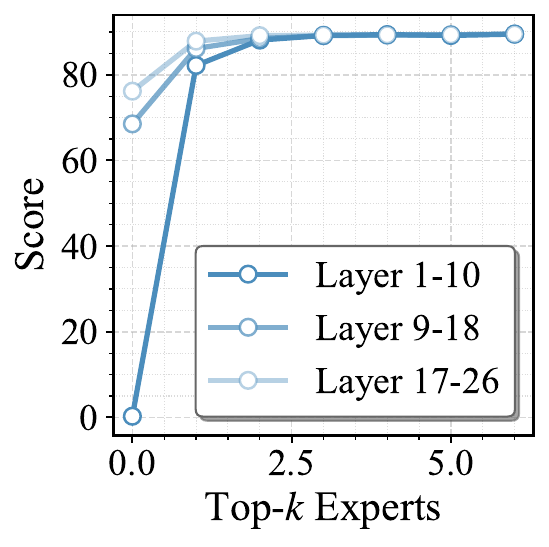}
            \end{subfigure}
       \end{minipage}
       \begin{minipage}[b]{0.326\linewidth}
            \centering
            \begin{subfigure}[tp!]{\linewidth}
            \centering
            \subcaption{MME~\cite{fu2023mme}}
            \includegraphics[width=\linewidth]{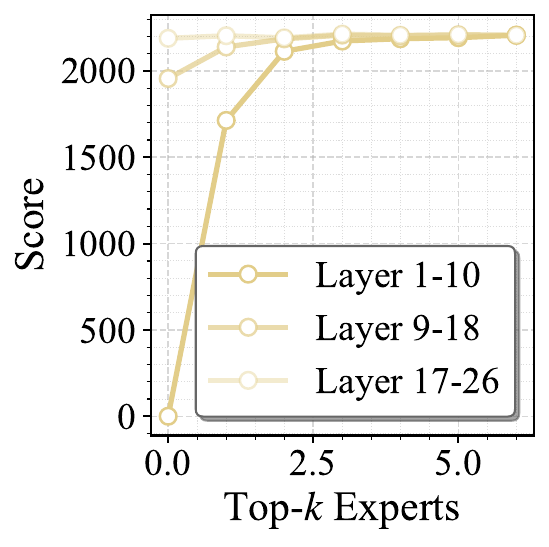}
            \end{subfigure}
       \end{minipage}
       \begin{minipage}[b]{0.326\linewidth}
            \centering
            \begin{subfigure}[tp!]{\linewidth}
            \centering
            \subcaption{VideoMMMU~\cite{hu2025videommmuevaluatingknowledgeacquisition}}
            \includegraphics[width=\linewidth]{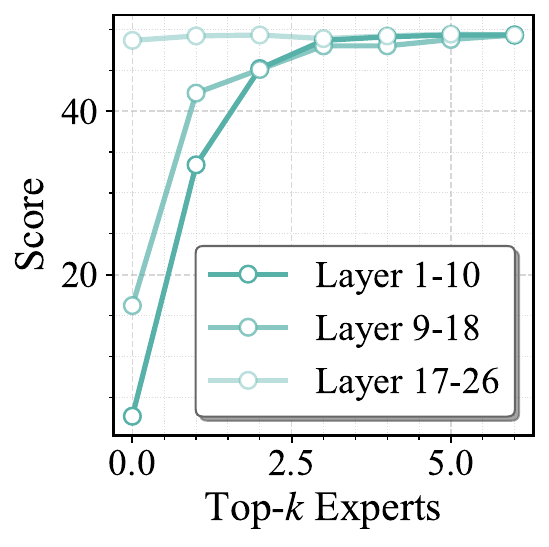}
            \end{subfigure}
       \end{minipage}
   \end{minipage}
   \vspace{-0.1in}
     \caption{Performance on image (\emph{i.e.}, (a)-(b)) and video (\emph{i.e.}, (c)) understanding tasks across various numbers of top-$k$ routed experts applied to different layer ranges for Kimi-VL-A3B-Instruct~\cite{team2025kimi}. The model has 64 routed experts for each FFN within the $1$-st to the $26$-th layers, and sets $k=6$ by default.}
    \label{fig:motivation_global}
    \vspace{-0.15in}
\end{figure}

\underline{\textit{Insight (i):}}  The observation yields a core design principle: With higher global contributions, experts in shallow-critical layers should be preserved; while experts in deeper, less influential ones can be skipped more aggressively.

\subsection{Modality Gap Matters}\label{sec:motivation_modality} %
Focusing on \textit{expert skipping} for MLLMs, we further examine the properties of modality-specific tokens with respect to the FFN layers. We first visualize the FFN input representations via t-SNE in Fig.~\ref{fig:motivation_modality} (\textit{Left}), which reveals a consistent distributional gap between text and vision tokens across layers. To quantify the effect of this modality disparity, we compute the cosine similarity between token representations before and after the FFNs. As shown in Fig.~\ref{fig:motivation_modality} (\textit{Middle}), FFNs induce a smaller effect on vision tokens (\emph{i.e.}, higher similarity for tokens pre- \emph{vs.} post-FFN), whereas text tokens undergo substantially larger updates. By tracking the angles between tokens and FFN weights in Fig.~\ref{fig:motivation_modality} (\textit{Right}), we attribute this phenomenon to their geometry: Vision tokens are more orthogonal to FFN weights (angles$\rightarrow90^\circ$), which alleviates the magnitude of their updates.

\begin{figure}[!ht]
\vspace{-0.1in}
   \centering
    \setlength{\abovecaptionskip}{0.2cm}
   \begin{minipage}[b]{\linewidth}
       \begin{minipage}[b]{0.326\linewidth}
            \centering
            \begin{subfigure}[tp!]{\linewidth}
            \centering
            \includegraphics[width=\linewidth]{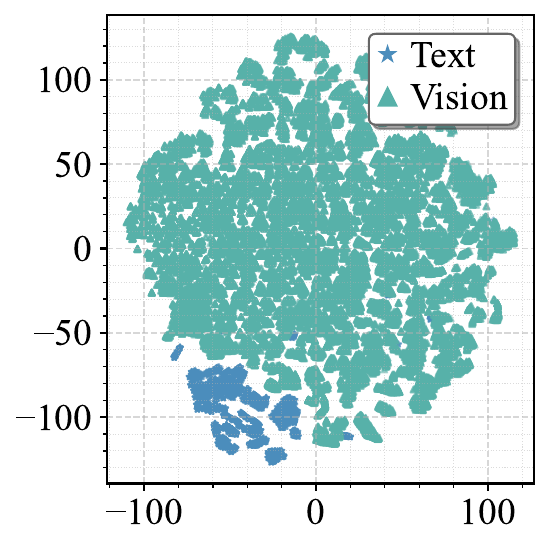}
            \end{subfigure}
       \end{minipage}
       \begin{minipage}[b]{0.326\linewidth}
            \centering
            \begin{subfigure}[tp!]{\linewidth}
            \centering
            \includegraphics[width=\linewidth]{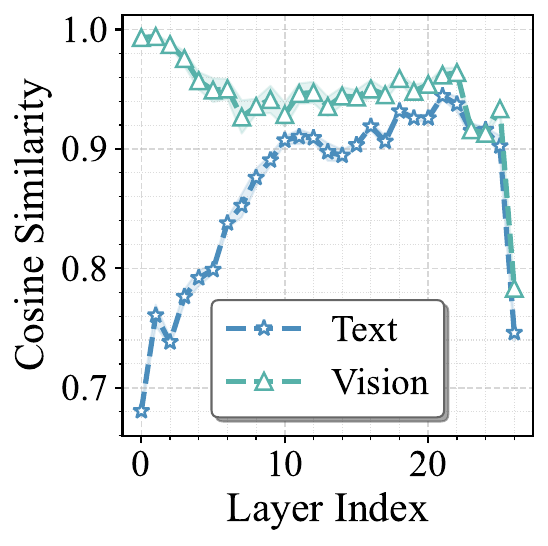}
            \end{subfigure}
       \end{minipage}
       \begin{minipage}[b]{0.326\linewidth}
            \centering
            \begin{subfigure}[tp!]{\linewidth}
            \centering
            \includegraphics[width=\linewidth]{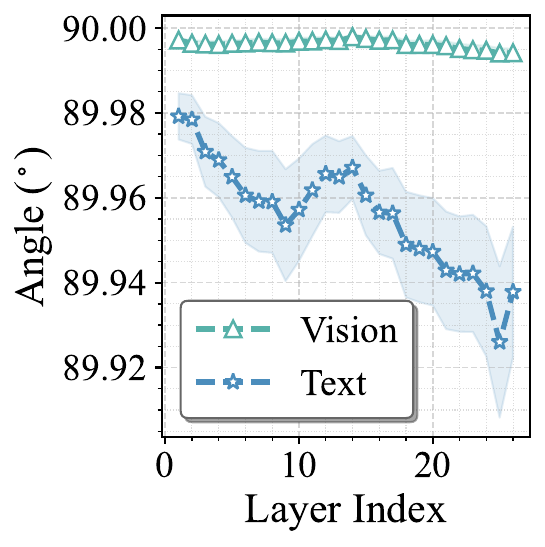}
            \end{subfigure}
       \end{minipage}
   \end{minipage}
   \vspace{-0.1in}
     \caption{(\textit{Left}) t-SNE~\cite{tsne} visualization of pre-FFN text/vision tokens across \textit{all} layers. (\textit{Middle}) Cosine similarity between pre-FFN and post-FFN text/vision tokens across layers. (\textit{Right}) Angle between text/vision tokens and weights across different FFN layers. Here, GQA~\cite{hudsom2019gqa} dataset is used as the model inputs, and the model is employed the same as that in Fig.~\ref{fig:motivation_global}.}
    \label{fig:motivation_modality}
    \vspace{-0.1in}
\end{figure}

\underline{\textit{Insight (ii):}} In a word, tokens from different modalities differ, and the magnitudes of updates by FFNs for tokens also vary across modalities. Intuitively, when deciding whether to skip the experts \emph{w.r.t.} the current token, we should account for these modality-specific differences. In the following, a modality-aware skipping policy is proposed for multimodal expert routing.

\begin{figure*}[!ht]
    \vspace{-0.3in}
    \centering
    \includegraphics[width=0.9\linewidth]{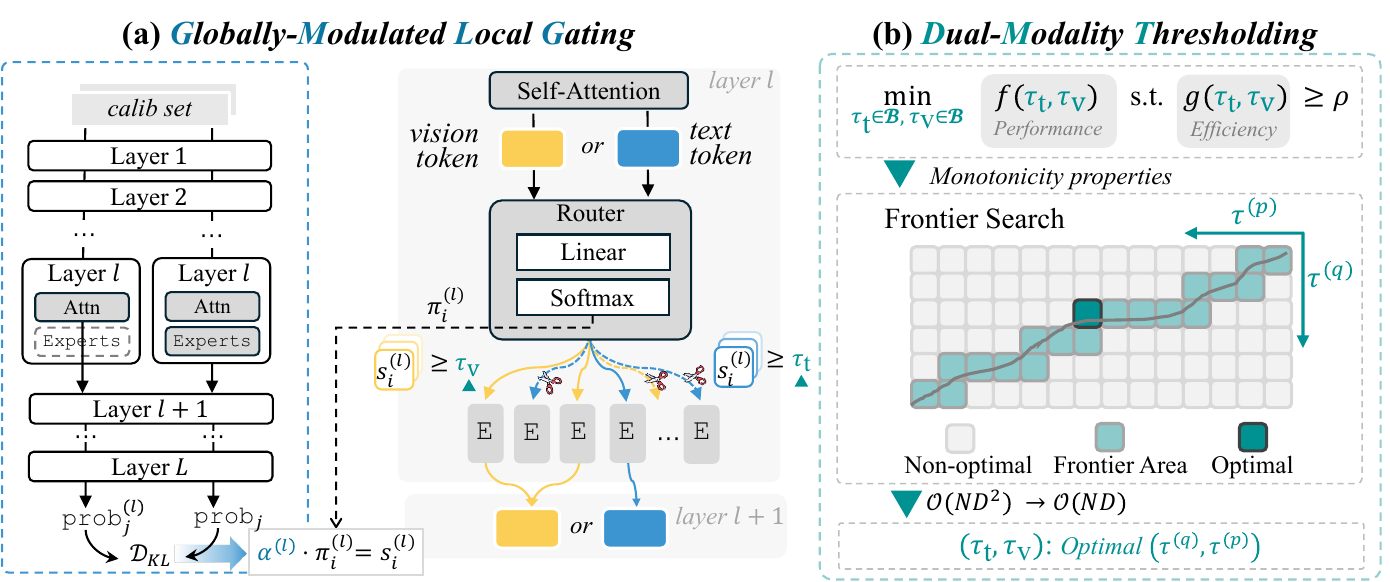}
    \vspace{-0.1in}
     \caption{Overview of \textit{MoDES}. At inference, use a text token (\emph{e.g.}, $\bl\blacksquare$ above) at the $l$-th FFN layer as an example. (\textit{a}) We compute importance scores $s^{(l)}_i$ ($i\in\{2, 4, M\}$) by combining the offline-calibrated globally-modulated factor $\bl{\alpha^{(l)}}$ with the local routing probability $\pi^{(l)}_i$. These scores evaluate the top-$k$ ($k=3$) routed experts for token $\bl\blacksquare$. (\textit{b}) We then apply a modality-specific threshold—$\mygreen{\tau_{\text{t}}}$ for text and $\mygreen{\tau_{\text{v}}}$ for vision—found by an efficient and effective \textit{frontier search}. Experts with scores below the threshold are skipped. This method significantly reduces computation while preserving performance for MoE MLLMs. ``\texttt{E}'' and ``\textit{calib set}'' denote the expert and $\mathcal{C}$ (Eq.~(\ref{eq:alpha})).}
    \label{fig:framework}
    \vspace{-0.15in}
\end{figure*}
\section{MoDES}\label{sec:modes}
  
Based on the above analyses, we propose \textit{MoDES} (\textit{Multimodal Dynamic Expert Skipping}), an efficient training-free framework composed of two key components, as illustrated in Fig.~\ref{fig:framework}:  
(\textit{i}) A \textit{globally-modulated local gating (GMLG)} (Sec.~\ref{sec:global}) mechanism that integrates a global and layer-level calibration with local routing probabilities to compute refined importance scores for top-$k$ experts; and  
(\textit{ii}) a \textit{dual-modality thresholding (DMT)}
(Sec.~\ref{sec:modality}) method that determines modality-specific skipping boundaries based on these importance scores.  
An efficiency–effectiveness search strategy is further introduced to optimize the threshold configuration under a given computational budget.

\subsection{Globally-Modulated Local Gating}
\label{sec:global}
In light of \underline{\textit{Insight (i)}} in Sec.~\ref{sec:motivation_global}, we present a \textit{globally-modulated local gating (GMLG)} mechanism, which combines the global contributions of experts with local routing behaviors to estimate expert importance for given tokens. During inference, experts in $\mathcal{S}^{(l)}$ (Eq.~(\ref{eq:aggregation})) with importance scores lower than the thresholds (defined in Sec.~\ref{sec:modality}) will be skipped. Specifically, for $\texttt{Expert}^{(l)}_i$ ($i\in \mathcal{S}^{(l)}$) with an input token $\mathbf{x}^{(l)}$, the importance score is defined as:
\begin{equation}
     s^{(l)}_i = \alpha^{(l)} \cdot \pi^{(l)}_i,
    \label{eq:important}
\end{equation}
where $\pi^{(l)}_i$ is the local routing probability (Eq.~(\ref{eq:logits})) that $\texttt{Expert}_i^{(l)}$ will be activated for $\mathbf{x}^{(l)}$. The globally-modulated factor $\alpha^{(l)}$ reflects the impact of experts in the layer on the final prediction, which is obtained by offline calibration. This $s^{(l)}_i$ accounts for both global and local contributions, yielding an accurate importance estimation.

To obtain $\alpha^{(l)}$, we calculate the Kullback-Leibler (KL) divergence between the output distribution of the original model and that of a counterpart where experts in the $l$-th layer are skipped:
\begin{equation}
    \alpha^{(l)} = \frac{1}{N}\sum^N_{j=1}\mathcal{D}_\mathrm{KL}\left(\texttt{prob}_j \,||\, \texttt{prob}^{(l)}_j\right),
    \label{eq:alpha}
\end{equation} 
where $N$ is the size of data (\emph{i.e.}, $\mathcal{C}=\{c_1,\ldots, c_N\}$) used for this calibration. $\texttt{prob}_j$ and $\texttt{prob}^{(l)}_j$ are the output probabilities for the $j$-th example of $\mathcal{C}$ from the original and modified models, respectively. This process quantifies the sensitivity of the model’s output to the removal of experts in certain layers, and $\alpha^{(l)}$ serves as a global importance weight reflecting their relative contributions. With the pre-computed $\alpha^{(l)}$, the final importance score $s^{(l)}_i$ can be obtained without additional overhead during inference.

\subsection{Dual-Modality Thresholding}
\label{sec:modality}
Building on \underline{\textit{Insight (ii)}} in Sec.~\ref{sec:motivation_modality}, we introduce a \textit{dual-modality thresholding (DMT)} method to adaptively determine modality-specific \textit{expert skipping} thresholds for MLLMs. We define two thresholds: $\tau_\text{t}$ for text tokens and $\tau_\text{v}$ for visual tokens, which control the degree of \textit{expert skipping} for each modality. This design considers the distinct behavior of tokens from different modalities, thereby allowing a tailored and effective skipping strategy.

To be specific, based on the importance scores (Eq.~(\ref{eq:important})) for the $l$-th layer, experts that should be skipped for the given token $\mathbf{x}^{(l)}$ are:
\begin{equation}
    \{\texttt{Expert}_i^{(l)} \mid s^{(l)}_i < \tau_\text{t} \cdot \mathbb{I}_{\text{t}} + \tau_\text{v} \cdot \mathbb{I}_\text{v}\},
    \label{eq:threshold}
\end{equation}
where $\mathbb{I}_\text{t}$ and $\mathbb{I}_\text{v}$ are text and vision token indicator functions for $\mathbf{x}^{(l)}$, respectively.

To find the optimal $\tau_\text{t}$ and $\tau_\text{v}$ that balance computational efficiency with model performance, we propose a \textit{frontier search} algorithm that effectively and efficiently determines these thresholds under an efficiency constraint. We first formulate the problem in the following.

\begin{algorithm}
\caption{Frontier search for optimal thresholds.}
\label{alg:frontier_search}
\small
\texttt{func} \textsc{FrontierSearch}($\mathcal{B}, \rho$)
\begin{algorithmic}[1]
    \Require 
    \Statex \space $\mathcal{B}$ --- Candidate set of thresholds $\{\tau^{(1)}, \dots, \tau^{(D)}\}$
    \Statex \space $\rho$ --- Target skipping ratio

    \State $\texttt{frontier}\gets \emptyset$
    \State $p\gets D$ %
    \For{$q=1$ \textbf{to} $D$}
        \While{$p\ge 1$ \textbf{and} $g(\tau^{(q)},\tau^{(p)}) \ge \rho$}
            \State $p \gets p-1$
        \EndWhile
        \State $p_{(q)} \gets p+1$ %
        \If{$p_{(q)} \le D$}
            \State Compute and save $f(\tau^{(q)}, \tau^{(p_{(q)})})$
            \State $\texttt{frontier} \gets \texttt{frontier} \cup \{(q,p_{(q)})\}$
        \EndIf
    \EndFor
    \State $(q^*, p^*) \gets \arg\,\min_{(q, p_{(q)})\in \texttt{frontier}} f(\tau^{(q)}, \tau^{(p_{(q)})})$
    \State \textbf{return} $(\tau^{(q^*)},\tau^{(p^*)})$
\end{algorithmic}

\end{algorithm}
\noindent\textbf{Problem definition.} For an MoE MLLM, the goal is to find the thresholds $\tau_\text{t}$ and $\tau_\text{v}$ that minimize the difference between the outputs of the original model and the \textit{expert-skipping} one, while satisfying a pre-defined target skipping ratio $\rho \in (0, 1)$. Hence, the problem can be expressed as:
\begin{equation}
\min_{\tau_\text{t} \in \mathcal{B}, \tau_\text{v}\in \mathcal{B}} f(\tau_\text{t}, \tau_\text{v})\quad  \text{s.t.} \quad g(\tau_\text{t}, \tau_\text{v}) \geq \rho,
\label{eq:problem}
\end{equation}
where $\mathcal{B}=\{\tau^{(1)},\ldots, \tau^{(D)}\}$ is the search grid set with $D$ candidates that satisfies $\tau^{(1)}< \tau^{(2)} <\ldots < \tau^{(D)}$. $f(\tau_\text{t}, \tau_\text{v})$ is the average KL divergence between the output distributions of the original model and the modified version, where experts are skipped according to Eq.~(\ref{eq:threshold}). $g(\tau_\text{t}, \tau_\text{v})$ is the fraction of experts that are skipped for the modified model.

\noindent\textbf{Frontier search.} We start with a monotonicity assumption:
\begin{assumption}
\textit{Holding other variables fixed, $f$ is non-decreasing in its respective arguments: If $q_1 \le q_2$, then $f(\tau^{(q_1)}, \tau^{(p)}) \le f(\tau^{(q_2)}, \tau^{(p)})$; and if $p_1 \le p_2$, then $f(\tau^{(q)}, \tau^{(p_1)}) \le f(\tau^{(q)}, \tau^{(p_2)})$.}
\label{ass:f_ass}
\end{assumption}
Intuitively, higher thresholds will skip more experts and degrade accuracy; hence, the assumption is reasonable. Obviously, $g$ is also non-decreasing in its respective arguments without any assumption. Given these monotonicity properties, we can search for a $\texttt{frontier}$ set $\{(q, p_{(q)})\}$ with a time complexity of $\mathcal{O}(ND)$~\footnote{We compute $f$ and $g$ on data $\mathcal{C}$ (with $N$ samples), which is also used in Eq.~(\ref{eq:alpha}).} through Lines 1-12 in Alg.~\ref{alg:frontier_search}. Here, $p_{(q)}$ for a given $q$ is defined as:
\begin{equation}
   p_{(q)} = \min \left\{\, p \in \{1, \dots, D\} \mid g( \tau^{(p)}, \tau^{(q)}) \geq \rho \,\right\}.
\end{equation}
We provide detailed proofs for the correctness of the search algorithm and its time complexity in the Appendix. Finally, as demonstrated in Alg.~\ref{alg:frontier_search}, the optimal thresholds $(\tau^{(q^*)}, \tau^{(p^*)})$, which lie in \texttt{frontier} (proofs can also be found in the Appendix), are obtained through Lines~13–14. Since all values of $f(\tau^{(q)}, \tau^{(p_{(q)})})$ are already computed by Line~9, this step takes less than a second.

Overall, our \textit{frontier search} algorithm achieves a time complexity of $\mathcal{O}(ND)$. In comparison, a naive solution involves an exhaustive search of all $(\tau_\text{t}, \tau_\text{v})$ pairs in $\mathcal{B}\times \mathcal{B}$, leading to a time complexity of $\mathcal{O}(ND^2)$. In practice, our method cuts the search time by a remarkable $\sim$45$\times$ (as detailed in Sec.~\ref{sec:efficiency}).

\section{Experiments}
\label{sec:experiments}

\begin{table*}[!ht]\setlength{\tabcolsep}{4pt}
\vspace{-0.3in}
 \renewcommand{\arraystretch}{1.}
  \centering
  \caption{Performance comparisons for Kimi-VL-A3B-Instruct~\cite{team2025kimi} across various expert skipping ratios. We mark the target $\rho$ (Eq.~(\ref{eq:problem})) and the practical skipping ratio $x\%$ (\emph{i.e.}, ``\textit{Skip $x\%$ Experts}'') in the table. For each method, we compute the score proportion relative to the default setting (\emph{i.e.}, $k=6$) across benchmarks, and then compute the average value in the ``Avg. (\%)'' column. For the COCO dataset, we report the CIDEr~\cite{vedantam2015ciderconsensusbasedimagedescription} score here. The best and second-best results are highlighted in \textbf{bold} and \underline{underlined} formats, respectively.}
  \vspace{-0.1in}
  \resizebox{\linewidth}{!}{
  \begin{tabu}[t!]{l|ccccccccccccc|c}
\toprule
\multirow{2}{*}{\textbf{Method}} & \multicolumn{8}{c}{\textbf{Image Understanding}} & \multicolumn{5}{c}{\textbf{Video Understanding}} & \multirow{2}{*}{\makecell{\textbf{Avg.}\\ \textbf{(\%)}}}\\
\cmidrule(lr){2-9}\cmidrule(lr){10-14} 
& \rotatebox{0}{TextVQA} & \rotatebox{0}{ChartQA} & \rotatebox{0}{MMStar} & \rotatebox{0}{MMBench} & \rotatebox{0}{MMVet} & \rotatebox{0}{MME} & \rotatebox{0}{RealWorldQA} & \rotatebox{0}{COCO} & \rotatebox{0}{MVBench} & \rotatebox{0}{EgoSchema} & \rotatebox{0}{VMME} & \rotatebox{0}{LVB} & \rotatebox{0}{VMMMU} &\\
\midrule
$k=6$ (\textit{Default}) & 88.70 & 89.48 & 49.89 & 83.16 & 66.33 & 2207 & 65.36 & 86.70 & 61.80 & 78.18 & 66.59 & 63.13 & 49.33 & 100.00 \\
\midrule
\multicolumn{15}{c}{\cellcolor[gray]{0.92} \textit{Skip $50\%$ Experts} ($\rho=0.48$)} \\
$k=3$ & 85.41 & 86.20 & 51.21 & 80.67 & 57.71 & 2065 & 63.53 & \underline{87.56} & 60.42 & 75.71 & 64.30 & 60.14 & 44.22 & 95.93\\
NAEE~\cite{lu2024not} & 86.14& 85.74& 50.82& 80.58& 60.81& 2084& 64.55& 85.33& 60.02& 75.81& 65.16& 60.27& 45.08 & 96.44\\
MC-MoE~\cite{huang2024mixture} & 86.28& 87.94& \textbf{51.61}& \underline{81.32}& \underline{62.54}& 2138& 63.82& 86.24& 60.39& 76.57& \underline{66.24}& 60.62& 46.26 & 97.69\\
DiEP~\cite{bai2025diep} & \underline{87.43}& \underline{88.32}& \underline{51.48}& 80.26& 60.41& \underline{2159}& \underline{64.74}& 87.43& \underline{61.06}& \underline{77.32}& 65.96& \underline{61.04}& \underline{47.83} & \underline{98.17}\\
\rowcolor{mycolor!30} MoDES (\textit{Ours}) & \textbf{88.18} & \textbf{89.08} & 49.65 & \textbf{83.16} & \textbf{65.09} & \textbf{2203} & \textbf{65.62} & \textbf{88.23} & \textbf{61.95} & \textbf{78.41} & \textbf{67.19} & \textbf{62.83} & \textbf{49.00} & \textbf{99.91}\\
\midrule
\multicolumn{15}{c}{\cellcolor[gray]{0.92} \textit{Skip $67\%$ Experts} ($\rho=0.65$)} \\
$k=2$ & 83.49 & 85.12 & \textbf{52.10} & 78.87 & 53.49 & 2022 & 63.79 & \textbf{92.61} & 59.35 & 70.80 & 62.15 & 57.67 & 41.44 & 93.88\\
NAEE~\cite{lu2024not} & 82.84& 85.29& 50.74& 77.31& 56.67& 2083& \underline{64.54}& 82.09& 59.68& 72.29& 63.74& 58.36& 43.68 & 94.03 \\
MC-MoE~\cite{huang2024mixture} & \underline{85.07}& \underline{86.32}& \underline{51.13}& 77.65& \underline{58.42}& \underline{2104}& 63.61& 84.23& 59.86& \underline{74.36}& \underline{64.22}& \underline{59.73}& \underline{45.21} & \underline{95.45}\\
DiEP~\cite{bai2025diep} & 84.21& 85.56& 50.76& \underline{78.94}& 57.05& 2087& 64.02& \underline{87.54}& \underline{60.02}& 72.97& 61.07& 58.45& 44.93 & 94.81 \\
\rowcolor{mycolor!30} MoDES (\textit{Ours}) & \textbf{85.57} & \textbf{88.24} & 49.25 & \textbf{82.73} & \textbf{60.78} & \textbf{2204} & \textbf{64.58} & 85.37 & \textbf{61.65} & \textbf{77.98} & \textbf{66.52} & \textbf{62.90} & \textbf{48.78} & \textbf{98.46}\\
\midrule
\multicolumn{15}{c}{\cellcolor[gray]{0.92} \textit{Skip $83\%$ Experts} ($\rho=0.80$)} \\
$k=1$ & 77.17 & 76.68 & 42.65 & 54.55 & 22.98 & 1647 & 54.38 & 77.37 & 51.10 & 37.23 & 50.52 & 43.83 & 24.56 & 71.60\\
NAEE~\cite{lu2024not} & 75.73 & 78.41 & 41.48 & 69.14 & 43.41 & 1827 & 60.32 & 72.35 & 58.41 & 57.28 & 53.49 & 49.68 & 42.64 & 82.81\\
MC-MoE~\cite{huang2024mixture} & 79.41 & \underline{80.25} & \underline{43.57} & 73.42 & \underline{50.37} & 2063 & \underline{62.54} & \underline{80.42} & 54.87 & \underline{63.56} & \underline{59.87} & \underline{54.39} & \underline{44.02} & \underline{88.32}\\
DiEP~\cite{bai2025diep} & \underline{82.32} & 78.31 & 42.47 & \underline{76.28} & 47.45 & \underline{2071} & 61.34 & 77.91 & \underline{59.15} & 61.27 & 57.49 & 52.41 & 43.81 & 87.58\\
\rowcolor{mycolor!30} MoDES (\textit{Ours}) & \textbf{82.38} & \textbf{84.20} & \textbf{46.68} & \textbf{81.44} & \textbf{60.46} & \textbf{2162} & \textbf{64.84} & \textbf{81.33} & \textbf{61.30} & \textbf{76.98} & \textbf{65.48} & \textbf{62.60} & \textbf{47.11} & \textbf{96.25}\\

\bottomrule
\end{tabu}
}
    \label{tab:compare-baseline}
    \vspace{-0.2in}
\end{table*}
\begin{table}[!ht]\setlength{\tabcolsep}{5.5pt}
 \renewcommand{\arraystretch}{1.}
  \centering
  \caption{Performance of combination with quantization. MoDES employs the quantization strategy in MC-MoE~\cite{huang2024mixture}: weight-only mixed-precision quantization for MoE-based FFNs and 4-bit weight-only quantization for other layers.} 
  \vspace{-0.1in}
  \resizebox{\linewidth}{!}{
  \begin{tabu}[t!]{l|c|cccccc}
\toprule
\textbf{Method} & \textbf{\#Bit} & \rotatebox{0}{ChartQA} & \rotatebox{0}{MME} & MMBench  & \rotatebox{0}{LVB} & \rotatebox{0}{VMMMU} \\
\midrule
\multicolumn{7}{l}{\textit{Kimi-VL-A3B-Instruct~\cite{team2025kimi}}} \\
\midrule

$k=6$ (\textit{Default}) & 16 & 89.48 & 2207 & 83.16  & 63.13 & 49.33\\

\midrule
\multicolumn{7}{c}{\cellcolor[gray]{0.92} \textit{Skip $67\%$ Experts} ($\rho=0.65$)} \\

MC-MoE~\cite{huang2024mixture} & 2.5 & 78.47 & 2036 & 68.84 & 54.46 & 41.92\\
\rowcolor{mycolor!30} MoDES (\textit{Ours}) & 2.5 & \textbf{81.23} & \textbf{2137} & \textbf{76.48} & \textbf{58.10} & \textbf{43.67} \\
MC-MoE~\cite{huang2024mixture} & 1.5 & 69.46 & 1728 & 62.18 & 42.87 & 38.45\\
\rowcolor{mycolor!30} MoDES (\textit{Ours}) & 1.5 &  \textbf{72.28} & \textbf{1899} & \textbf{68.57} & \textbf{48.14} & \textbf{40.06}\\

\midrule
\multicolumn{7}{l}{\textit{Qwen3-VL-MoE-30B-A3B-Instruct~\cite{qwen3_vl_moe_doc}}}\\
\midrule
$k=8$ (\textit{Default}) & 16 & 85.08 & 2500 & 86.60  & 55.42 & 47.11 \\

\midrule
\multicolumn{7}{c}{\cellcolor[gray]{0.92} \textit{Skip $75\%$ Experts} ($\rho=0.73$)} \\

MC-MoE~\cite{huang2024mixture} & 2.5 & 76.36 & 2084 & 79.62 & 51.85 & 42.06\\
\rowcolor{mycolor!30} MoDES (\textit{Ours}) & 2.5 & \textbf{78.24} & \textbf{2281} & \textbf{81.34} & \textbf{53.63} & \textbf{46.28}\\
MC-MoE~\cite{huang2024mixture} & 1.5 & 70.42 & 1968 & 73.18 & 46.08 & 36.94\\
\rowcolor{mycolor!30} MoDES (\textit{Ours}) & 1.5 & \textbf{73.42} & \textbf{2113} & \textbf{75.54} & \textbf{47.32} & \textbf{42.01}\\

\bottomrule
\end{tabu}
}
    \label{tab:quant}
    \vspace{-0.2in}
\end{table}

\subsection{Setups}\label{sec:imple}

\noindent\textbf{Models and datasets.} We choose 3 series of MoE MLLMs to evaluate \textit{MoDES}: Kimi-VL~\cite{team2025kimi},
Qwen3-VL-MoE~\cite{qwen3_vl_moe_doc}, and InternVL-3.5~\cite{wang2025internvl3}. We use 8 zero-shot evaluation tasks for image understanding:
TextVQA$_{\text{val}}$~\cite{singh2019textvqa}, ChartQA~\cite{masry2022chartqa}, MMStar~\cite{chen2024rightwayevaluatinglarge}, MMBench$_{\text{dev, en}}$~\cite{liu2024mmbenchmultimodalmodelallaround},
MMVet~\cite{yu2023mmvet}, MME~\cite{fu2023mme}, RealWorldQA~\cite{realworldQA}, and COCO2017-Cap$_{\text{val}}$~\cite{lin2015microsoft} (COCO). For video understanding tasks, we adopt 5 benchmarks: MVBench~\cite{li2024mvbenchcomprehensivemultimodalvideo}, EgoSchema~\cite{mangalam2023egoschemadiagnosticbenchmarklongform}, VideoMME~\cite{fu2025video} (VMME), LongVideoBench$_{\text{val,v}}$~\cite{wu2024longvideobenchbenchmarklongcontextinterleaved} (LVB), and VideoMMMU~\cite{hu2025videommmuevaluatingknowledgeacquisition} (VMMMU). \texttt{lmms-eval}~\cite{zhang2024lmmsevalrealitycheckevaluation} is utilized to perform the above evaluation. For MMBench and MMVet, we use DeepSeek-V3.1~\cite{deepseekai2024deepseekv3technicalreport} to rate the generated texts.

\noindent\textbf{Baselines.} As there is no \textit{expert skipping} baselines for MLLMs and previous methods for LLMs only consider models with top-$2$ routing in practice, we re-implement and adjust them to top-$k$ ($k>2$) settings for MLLMs: For the $l$-th layer, NAEE~\cite{lu2024not} originally skips the top-$2$ expert if $\pi^{(l)}_{\text{top-}2}<\beta^{(l)}\cdot \pi^{(l)}_{\text{top-}1}$, where $\pi^{(l)}_{\text{top-}1}$ and $\pi^{(l)}_{\text{top-}2}$ denotes the top-$1$ and top-$2$ routing probabilities (Eq.~(\ref{eq:logits})). $\beta^{(l)}$ is a hyperparameter. Here, we adapt this strategy, referring to the Appendix of NAEE, to a more general top-$k$ scenario. Specifically,  top-$i$ to top-$k$ experts are skipped if $\sum_{u=i}^k\pi^{(l)}_{\text{top-}u}<\beta^{(l)}\cdot\sum^k_{v=1}\pi^{(l)}_{\text{top-}v}$. We also apply similar adjustments for MC-MoE~\cite{huang2024mixture} and DiEP~\cite{bai2025diep}, which build on top of NAEE. To be noted, without a specific claim, we adopt only the \textit{expert skipping} component of these works to enable a fair comparison. Moreover, we also compare our method with \textit{expert skipping} guided by directly reducing the value $k$ of top-$k$ routing.

\noindent\textbf{Implementation.} We employ 1024 samples randomly picked from the GQA~\cite{hudsom2019gqa} dataset to calibrate $\alpha^{(l)}$ (Eq.~(\ref{eq:alpha})) and search optimal $(\tau_\text{t}, \tau_\text{v})$ (Eq.~(\ref{eq:threshold})). The search space $\mathcal{B}$ is given by $D=100$ grid points sampled in $(0, 1)$. More implementation details can be found in the Appendix.

\subsection{Evaluation}\label{sec:evaluation}

\noindent\textbf{Comparison with baselines.} We benchmark MoDES against baselines on Kimi-VL-A3B-Instruct~\cite{team2025kimi}. As shown in Tab.~\ref{tab:compare-baseline}, prior methods, such as NAEE~\cite{lu2024not}, MC-MoE~\cite{huang2024mixture}, and DiEP~\cite{bai2025diep}, struggle to balance performance and efficiency, especially at high \textit{expert-skipping} ratios ($\geq$67\%). Specifically, these baselines incur an average accuracy drop of more than 11\% when skipping 83\% of experts during inference. We argue that these declines arise because they rely solely on \textit{intra-layer} routing logits (Eq.~(\ref{eq:logits})) to determine the skipping schedule and are originally designed for unimodal LLMs. By contrast, our method, which considers both the impact of \textit{expert skipping} on the final output and the modality gap in MLLMs (Sec.~\ref{sec:motivation_modality}), executes only 13\% of experts, while preserving 96.25\% of the full model’s average accuracy. Moreover, even at a lower skipping ratio of 50\%, our approach still surpasses DiEP and MC-MoE by 1.74\% and 2.22\%, respectively. These findings validate the superiority of our method across different skipping ratios compared with existing SOTA approaches. In addition, on some benchmarks (\emph{e.g.}, RealWorldQA~\cite{realworldQA} and VideoMME~\cite{fu2025video}), using MoDES to skip redundant experts not only prevents degradation but also improves accuracy, suggesting that certain experts are not merely redundant but may actively interfere with inference.

\begin{table*}[!ht]\setlength{\tabcolsep}{4pt}
\vspace{-0.3in}
 \renewcommand{\arraystretch}{1.}
  \centering
  \caption{Performance comparisons across different backbones. InternVL series employs Qwen3~\cite{yang2025qwen3} and GPT-OSS~\cite{openai2025gptoss120bgptoss20bmodel} as LLM backbones for 30B and 20B models, respectively. The number of experts for each layer of models from upper to lower is 128, 128, and 32.} 
  \vspace{-0.1in}
  \resizebox{\linewidth}{!}{
  \begin{tabu}[t!]{l|ccccccccccccc|c}
\toprule
\multirow{2}{*}{\textbf{Method}} & \multicolumn{8}{c}{\textbf{Image Understanding}} & \multicolumn{5}{c}{\textbf{Video Understanding}} & \multirow{2}{*}{\makecell{\textbf{Avg.}\\ \textbf{(\%)}}}\\
\cmidrule(lr){2-9}\cmidrule(lr){10-14} 
& \rotatebox{0}{TextVQA} & \rotatebox{0}{ChartQA} & \rotatebox{0}{MMStar} & \rotatebox{0}{MMBench} & \rotatebox{0}{MMVet} & \rotatebox{0}{MME} & \rotatebox{0}{RealWorldQA} & \rotatebox{0}{COCO} & \rotatebox{0}{MVBench} & \rotatebox{0}{EgoSchema} & \rotatebox{0}{VMME} & \rotatebox{0}{LVB} & \rotatebox{0}{VMMMU} &\\
\midrule

\multicolumn{15}{l}{\textit{Qwen3-VL-MoE-30B-A3B-Instruct~\cite{qwen3_vl_moe_doc}}} \\
\midrule
$k=8$ (\textit{Default}) & 83.41 & 85.08 & 59.67 & 86.60 & 69.68 & 2500 & 66.80 & 80.37 & 64.67 & 62.45 & 54.89 & 55.42 & 47.11 & 100.00 \\
\midrule
\multicolumn{15}{c}{\cellcolor[gray]{0.92} \textit{Skip $88\%$ Experts} ($\rho=0.85$)} \\
$k=1$ & 60.71 & 52.16 & 31.63 & 54.90 & 28.07 & 1590 & 52.42 & 45.64 & 41.51 & 32.52 & 39.78 & 42.41 & 12.51 & 60.11\\
NAEE~\cite{lu2024not} & 72.41& 65.83& 48.88& 73.62& 54.52& 1984& 58.62& 60.37& 50.24& 49.77& 44.48& 45.59& 35.57 & 80.60 \\
MC-MoE~\cite{huang2024mixture} & \underline{74.87} & \underline{71.43}& 50.74& \underline{75.42}& \underline{61.35}& \underline{2168}& 60.41& \underline{68.15}& 56.60& 51.84& \underline{52.51}& \underline{47.22}& \underline{37.41} & \underline{86.66} \\
DiEP~\cite{bai2025diep} & 73.46& 70.51& \underline{53.28}& 73.21& 58.64& 2074& \underline{63.41}& 62.89& \underline{57.21}& \underline{53.61}& 50.78& 46.13& 34.79 & 85.30 \\
\rowcolor{mycolor!30} MoDES (\textit{Ours}) & \textbf{80.97} & \textbf{78.84} & \textbf{58.18} & \textbf{85.57} & \textbf{67.75} & \textbf{2403} & \textbf{64.58} & \textbf{74.66} & \textbf{62.98}& \textbf{62.04} & \textbf{55.26} & \textbf{55.50} & \textbf{46.56} & \textbf{97.33}\\

\midrule
\multicolumn{15}{l}{\textit{InternVL-3.5-30B-A3B-HF~\cite{wang2025internvl3}}} \\
\midrule
$k=8$ (\textit{Default}) & 85.76 & 84.08 & 62.49 & 83.81 & 69.93 & 2312 & 64.77 & 69.30 & 68.92 & 60.49 & 58.07 & 57.64 & 45.11 & 100.00\\
\midrule
\multicolumn{15}{c}{\cellcolor[gray]{0.92} \textit{Skip $88\%$ Experts} ($\rho=0.85$)} \\
$k=1$ & 58.49 & 46.24 & 42.27 & 51.74 & 35.05 & 1683 & 51.44 & 26.01 & 31.99 & 34.47 & 35.26 & 37.40 & 24.27 & 59.63\\
NAEE~\cite{lu2024not} & 66.24& 68.32& 50.14& 64.37& 49.52& 1802& 55.23& 50.64& 54.78& 50.25& 48.69& 47.42& 37.27 & 78.88\\
MC-MoE~\cite{huang2024mixture} & \underline{70.41}& \underline{73.49}& 56.14& \underline{64.38} & \underline{72.41}& \underline{1972}& \underline{57.49}& \underline{60.12}& \underline{58.97}& \underline{52.31}& \underline{49.72}& \underline{48.31}& \underline{40.06} & \underline{86.20} \\
DiEP~\cite{bai2025diep} & 69.37& 71.84& \underline{57.21}& 63.19& 65.32& 1838& 56.38& 55.78& 56.26& 51.48& 48.94& 47.26& 38.18 & 83.26\\
\rowcolor{mycolor!30} MoDES (\textit{Ours}) & \textbf{80.58} & \textbf{82.00} & \textbf{61.20} & \textbf{81.67} & \textbf{67.80} & \textbf{2222} & \textbf{61.73} & \textbf{65.16} & \textbf{68.65} & \textbf{60.79} & \textbf{57.63} & \textbf{54.49} & \textbf{44.33} & \textbf{97.03}\\

\midrule
\multicolumn{15}{l}{\textit{InternVL-3.5-GPT-OSS-20B-A4B-Preview-HF~\cite{wang2025internvl3}}} \\
\midrule
$k=4$ (\textit{Default}) & 80.20 & 90.64 & 57.64 & 79.78 & 69.68 & 2270 & 61.63 & 70.61 & 67.65 & 58.79 & 53.93 & 54.65 & 43.79 & 100.00\\
\midrule
\multicolumn{15}{c}{\cellcolor[gray]{0.92} \textit{Skip $75\%$ Experts} ($\rho=0.73$)} \\
$k=1$ & 68.74 & 79.72 & 45.77 & 67.63 & 48.49 & 1833 & 53.20 & 60.70 & 56.95 & 49.40 & 44.04 & 44.28 & 41.66 & 77.58 \\
NAEE~\cite{lu2024not} & 73.89& 82.34& 44.89& 71.59& 54.97& 2017& \textbf{63.46}& 59.73& 51.25& 46.21& 47.83& 45.48& 42.08 & 86.79 \\
MC-MoE~\cite{huang2024mixture} & 76.49& 84.53& 46.25& 73.68& 56.83& \underline{2137}& \underline{61.07}& 60.42& \underline{60.06}& \underline{50.28}& 48.37& 46.68& \underline{42.89} & 89.91 \\
DiEP~\cite{bai2025diep} & \underline{77.31}& \underline{86.24}& \underline{48.18}& \underline{74.26}& \underline{58.07}& 2109& 60.25& \underline{62.08}& 54.18& 49.83& \underline{49.42}& \underline{47.91}& 42.31 & \underline{90.07}\\
\rowcolor{mycolor!30} MoDES (\textit{Ours}) & \textbf{77.93} & \textbf{89.60} & \textbf{56.48} & \textbf{78.14} & \textbf{66.33} & \textbf{2206} & 60.64 & \textbf{68.32} & \textbf{66.60} & \textbf{57.95} & \textbf{53.59} & \textbf{53.68} & \textbf{43.13} & \textbf{97.89}\\

\bottomrule
\end{tabu}
}
    \label{tab:compare-backbone}
    \vspace{-0.15in}
\end{table*}

\noindent\textbf{Combination with quantization.} We conduct experiments to demonstrate the high compatibility of our MoDES with model quantization. As shown in Tab.~\ref{tab:quant} (see the performance without quantization for \textit{expert skipping} in Tab.~\ref{tab:compare-baseline} and the Appendix), quantization causes a much smaller performance drop for MoDES than for MC\mbox{-}MoE. For instance, on Kimi-VL-A3B-Instruct with a $\sim$10.67$\times$ compression ratio (\emph{i.e.}, 1.5 bits), quantization reduces MoDES's performance by 17.30\%, compared with $>$20\% for MC-MoE. In addition, 2.5-bit quantization keeps MoDES more than 90\% of the original model performance. Remarkably, for Qwen3-VL-MoE-30B-A3B-Instruct, it retains 94.43\% performance, whereas 2.5-bit MC-MoE retains 89.58\%. In future work, we will explore combining MoDES with other orthogonal techniques, such as pruning and distillation, to further reduce the computational demands of MoE MLLMs.

\noindent\textbf{Comparison across backbones.} In Tab.~\ref{tab:compare-backbone}, we evaluate our method across multiple backbones. On the powerful Qwen3-VL-MoE-30B-A3B-Instruct model~\cite{qwen3_vl_moe_doc}, our approach retains 97.33\% of the original performance at an aggressive skipping ratio of 88\%. Moreover, across backbones, our method outperforms other skipping strategies by more than 5\% points in average accuracy. Taken together, these results highlight the effectiveness and universality of our technique in identifying redundant experts for tokens of different modalities and across different layers. In addition, we provide comparisons across different skipping ratios for these models in the Appendix, where our method consistently delivers higher accuracy at matched skipping ratios. We further exhibit some qualitative visual reasoning examples in the Appendix to comprehensively demonstrate the superiority of our method.

\subsection{Efficiency Discussion}\label{sec:efficiency}

\begin{figure}[!ht]
\vspace{-0.2in}
    \centering
    \includegraphics[width=0.95\linewidth]{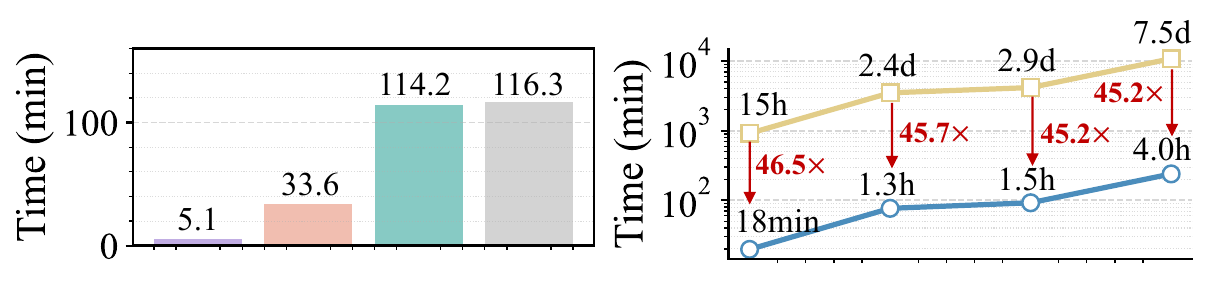}
    \vspace{-0.15in}
     \caption{(\textit{Left}) $\alpha^{(l)}$ calibration time. (\textit{Right}) Search time of \textit{frontier search} (\bl{blue}) \emph{vs.} \textit{naive search} (\myyellow{yellow}). The bars/markers from \textit{left} to \textit{right} are for Kimi-VL-A3B-Instruct~\cite{team2025kimi}, Qwen3-VL-MoE-30B-A3B-Instruct~\cite{qwen3_vl_moe_doc}, InternVL-3.5-30B-A3B-HF~\cite{wang2025internvl3}, and InternVL-3.5-GPT-OSS-20B-A4B-Preview-HF~\cite{wang2025internvl3}.}
    \label{fig:calib_search}
    \vspace{-0.1in}
\end{figure}

\noindent\textbf{Calibration and search efficiency.} As illustrated in Fig.~\ref{fig:calib_search}, we evaluate the calibration and search times of MoDES for MoE MLLMs with $\geq$20B parameters on $8\times$H200 GPUs. It is important to note that since InternVL-3.5-GPT-OSS-20B-A4B-Preview-HF~\cite{wang2025internvl3} in the \texttt{transformers}~\cite{wolf-etal-2020-transformers} library supports only naive attention computation, its time consumption is significantly higher compared to the same-sized Kimi-VL-A3B-Instruct, which uses \texttt{flash-attention2}~\cite{dao2023flashattention2fasterattentionbetter}. As observed from the other models, MoDES processes 20-30B MoE MLLMs (\emph{i.e.}, calibration + search) in 20 minutes to under 4 hours, demonstrating high efficiency. Furthermore, compared to \textit{naive search} with $\mathcal{O}(ND^2)$ complexity, our \textit{frontier search} with $\mathcal{O}(ND)$ significantly reduces the search time by $\sim$45$\times$. In terms of performance, we benchmarked \textit{naive search} for Kimi-VL-A3B-Instruct~\cite{team2025kimi} with an 83\% \textit{expert skipping} ratio and found nearly identical average performance with \textit{frontier search} (96.24\% \emph{vs.} 96.25\%). This result helps confirm the correctness of our Alg.~\ref{alg:frontier_search}.

\begin{figure}[!ht]
\vspace{-0.15in}
   \centering
    \setlength{\abovecaptionskip}{0.2cm}
   \begin{minipage}[b]{\linewidth}
       \begin{minipage}[b]{\linewidth}
            \centering
            \begin{subfigure}[tp!]{\textwidth}
            \centering
            \includegraphics[width=\linewidth]{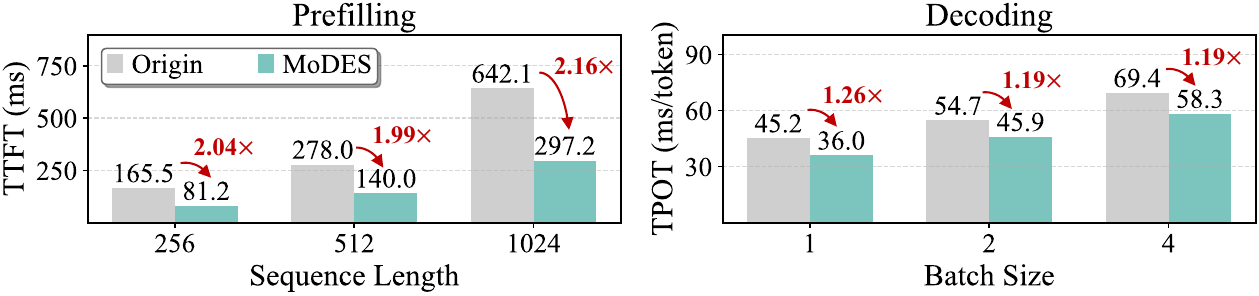}
            \end{subfigure}
       \end{minipage}
       \begin{minipage}[b]{\linewidth}
            \centering
            \begin{subfigure}[tp!]{\linewidth}
            \centering
            \includegraphics[width=\linewidth]{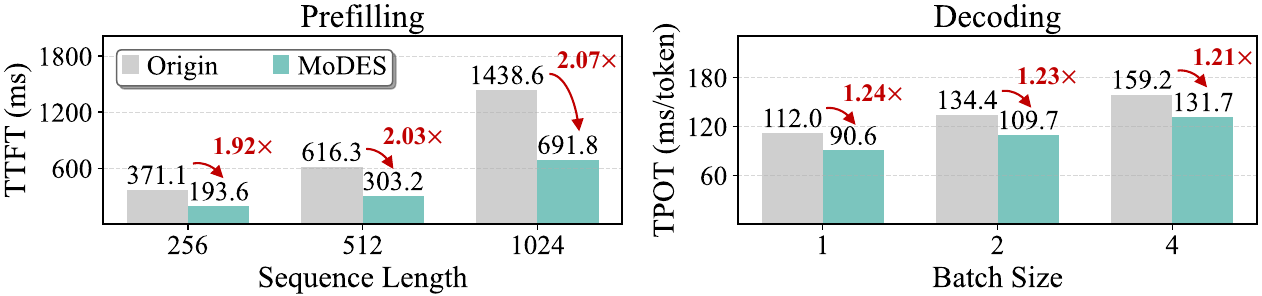}
            \end{subfigure}
       \end{minipage}
   \end{minipage}
   \vspace{-0.1in}
     \caption{Inference speed for (\textit{Upper}) Kimi-VL-A3B-Instruct~\cite{team2025kimi} and (\textit{Lower}) Qwen3-VL-MoE-30B-A3B-Instruct~\cite{qwen3_vl_moe_doc} on a single H200 GPU. The \textit{expert skipping} ratios for the former and the latter are 83\% and 88\%, respectively. The batch size for prefilling is 8, and the sequence length for decoding is 1024.}
    \label{fig:inference}
    \vspace{-0.15in}
\end{figure}

\noindent\textbf{Inference efficiency.} Next, we study the practical inference speedup. As shown in Fig.~\ref{fig:inference}, MoDES attains an $\sim$2$\times$ speedup in the prefill phase compared with the original model. In the decoding phase, it still delivers a $\sim$1.2$\times$ speedup. The smaller ratio during decoding likely arises because: (\textit{i}) MoDES primarily reduces computation in MoE layers, while decoding remains memory-bound; and (\textit{ii}) only text tokens are processed during decoding, which leads to lower \textit{expert skipping} ratios (Sec.~\ref{sec:vis}). In addition, baselines like DiEP~\cite{bai2025diep} use offline calibration to select hyperparameters, so their inference overhead is negligible. Under the same skipping ratios, their speedup ratios are similar to ours with $<$1\% difference. Despite this, our method outperforms them across benchmarks by a clear margin (Sec.~\ref{sec:evaluation}).

\subsection{Ablation Studies}
In this section, we employ Kimi-VL-A3B-Instruct~\cite{team2025kimi}, and the settings are the same as those in Sec.~\ref{sec:imple} without specific claims. More ablations can be found in the Appendix.

\noindent\textbf{Effect of each component.} We evaluate each component of MoDES and use a single threshold $\tau$ \emph{w.r.t.} $s^{(l)}_i=\pi_i^{(l)}$ (denoted as ``Thresholding'') with a grid search as our baseline. As shown in Tab.~\ref{tab:ablation}, GMLG, which incorporates both global and local contributions, significantly enhances both Thresholding and DMT. Moreover, by applying different thresholds for different modalities, DMT outperforms Thresholding by a large margin. These results underscore the importance of the two key insights discussed in Sec.~\ref{sec:motivation}, highlighting the substantial contributions of both GMLG and DMT. Remarkably, performance improvements derived from GMLG and DMT increase as the skipping ratio grows.

\begin{table}[!ht]\setlength{\tabcolsep}{5.5pt}
\vspace{-0.1in}
 \renewcommand{\arraystretch}{1.}
  \centering
  \caption{Ablation results for each component of MoDES. ``Thresholding'' means we employ a single threshold $\tau$ for both modalities and adopt a grid search for the optimal $\tau$. For Thresholding and DMT, we set $s^{(l)}_i=\pi_i^{(l)}$, instead of using Eq.~(\ref{eq:important}).} 
  \vspace{-0.1in}
  \resizebox{\linewidth}{!}{
  \begin{tabu}[t!]{l|cccccc}
\toprule
\textbf{Method} & \rotatebox{0}{ChartQA} & \rotatebox{0}{MME} & MMBench  & \rotatebox{0}{LVB} & \rotatebox{0}{VMMMU} \\
\midrule

$k=6$ (\textit{Default}) &  89.48 & 2207 & 83.16  & 63.13 & 49.33\\

\midrule

\multicolumn{7}{c}{\cellcolor[gray]{0.92} \textit{Skip $67\%$ Experts} ($\rho=0.65$)} \\
\midrule
Thresholding & 85.48 & 2030 & 77.67 & 57.97 & 45.56\\ 
Thresholding w/ GMLG & \underline{87.64} & \underline{2172} & 79.46 & 60.24 & 46.48\\
DMT & 87.47 & 2158 & \underline{81.07} & \underline{61.26} & \underline{46.88}\\
\rowcolor{mycolor!30}DMT w/ GMLG (\textit{Ours}) & \textbf{88.24} & \textbf{2204} & \textbf{82.73} & \textbf{62.90} & \textbf{48.78} \\

\midrule
\multicolumn{7}{c}{\cellcolor[gray]{0.92} \textit{Skip $83\%$ Experts} ($\rho=0.80$)} \\
\midrule
Thresholding & 76.74 & 1956 & 65.48 & 54.67 & 40.33\\ 
Thresholding w/ GMLG & 79.28 & \underline{2107} & 75.19 & 60.02 & 43.87\\
DMT & \underline{82.94} & 2081 & \underline{79.42} & \underline{61.16} & \underline{45.08} \\
\rowcolor{mycolor!30}DMT w/ GMLG (\textit{Ours}) & \textbf{84.20} & \textbf{2162} & \textbf{81.44} & \textbf{62.60} & \textbf{47.11} \\
\bottomrule
\end{tabu}
}
    \label{tab:ablation}
    \vspace{-0.15in}
\end{table}

\begin{figure}[!ht]
\vspace{-0.15in}
\centering
\begin{minipage}{0.6\linewidth}
    \centering
    \captionof{table}{Ablation results of using 3 different datasets for both calibration and frontier search (C\&S).}
    \vspace{-0.1in}
    \resizebox{\linewidth}{!}{
    \begin{tabu}[t!]{l|ccc}
\toprule
\textbf{C\&S} & \cellcolor{mycolor!30}\rotatebox{0}{GQA} & \rotatebox{0}{COCO} & \rotatebox{0}{VMMMU} \\
\midrule
\multicolumn{4}{c}{\cellcolor[gray]{0.92} \textit{Skip $83\%$ Experts} ($\rho=0.80$)} \\
\midrule
GQA & \cellcolor{mycolor!30}\textbf{62.68} & \underline{62.65} & 62.63 \\
COCO & \cellcolor{mycolor!30}\underline{81.33} & \textbf{81.72} & 80.72\\
VMMMU & \cellcolor{mycolor!30}\underline{47.11} & \textbf{47.67} & \textbf{47.67}\\
ChartQA & \cellcolor{mycolor!30}\underline{84.20} & \textbf{86.56} & 83.46 \\
MMBench & \cellcolor{mycolor!30}\underline{81.44} & 79.38 & \textbf{81.87}\\
MME & \cellcolor{mycolor!30}\underline{2162} & \textbf{2138} & 2136 \\
LVB & \cellcolor{mycolor!30}\underline{62.60} & 62.30 & \textbf{62.75}\\

\bottomrule
\end{tabu}
}
    \label{tab:data}
\end{minipage}%
\hfill
\begin{minipage}{0.38\linewidth}
    \centering
    \includegraphics[width=\linewidth]{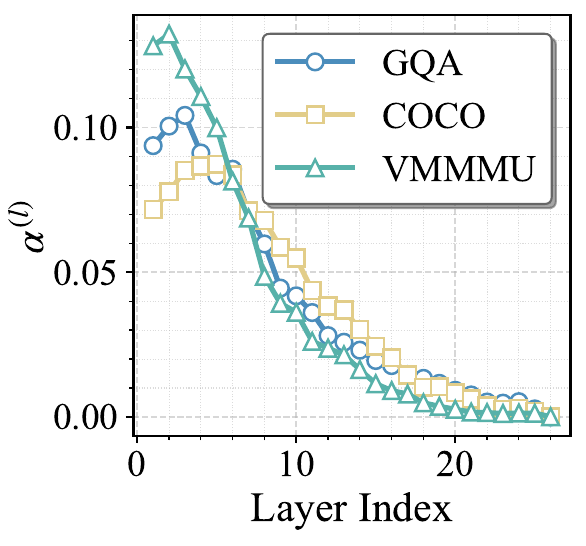} %
    \vspace{-0.3in}
    \caption{Visualization results of global contributions $\alpha^{(l)}$ (Eq.~(\ref{eq:alpha})) across layers and various datasets.}
    \label{fig:data}
\end{minipage}
\vspace{-0.15in}
\end{figure}

\noindent\textbf{Choice of data.} We also investigate the effect of different datasets with randomly sampled 1024 examples on MoDES. In Fig.~\ref{fig:data}, the trends of $\alpha^{(l)}$ across datasets are similar, with shallow layers having larger values than deep layers. This aligns with our insight in Sec.~\ref{sec:motivation_global}, where experts in shallow layers contribute more to the final outputs. Additionally, the performance is also consistent across datasets, as shown in Tab.~\ref{tab:data}. These results indicate that MoDES is robust and not sensitive to the choice of dataset.

\subsection{Visualization Analysis}\label{sec:vis}

\begin{figure}[!ht]
\vspace{-0.1in}
   \centering
    \setlength{\abovecaptionskip}{0.2cm}
   \begin{minipage}[b]{\linewidth}
       \begin{minipage}[b]{0.498\linewidth}
            \centering
            \begin{subfigure}[tp!]{\textwidth}
            \centering
            \includegraphics[width=\linewidth]{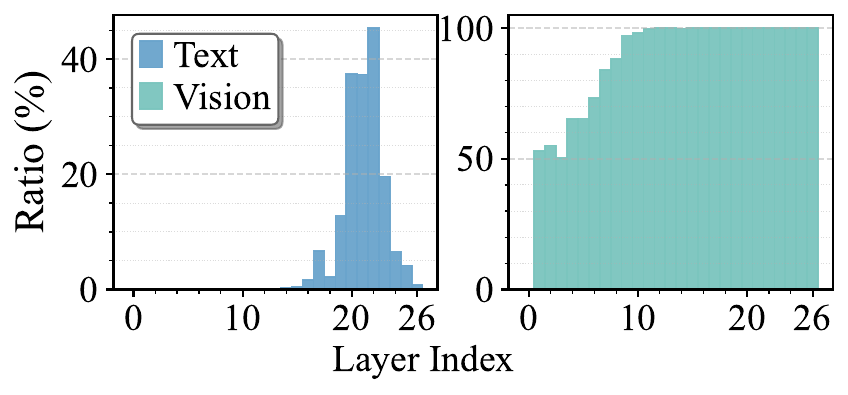}
            \end{subfigure}
       \end{minipage}\hfill
       \begin{minipage}[b]{0.498\linewidth}
            \centering
            \begin{subfigure}[tp!]{\linewidth}
            \centering
            \includegraphics[width=\linewidth]{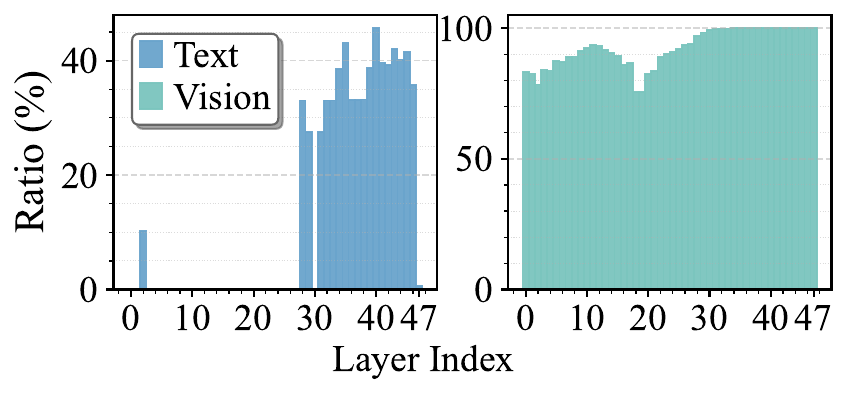}
            \end{subfigure}
       \end{minipage}
   \end{minipage}
   \vspace{-0.1in}
     \caption{Visualization of \textit{expert skipping} ratios (\%) across modalities and layers on 13 benchmarks (Sec.~\ref{sec:imple}). The \textit{left} subfigure is for Kimi-VL-A3B-Instruct~\cite{team2025kimi} and the \textit{right} subfigure is for Qwen3-VL-MoE-30B-A3B-Instruct~\cite{qwen3_vl_moe_doc}. The \textit{overall} skipping ratios for the former and the latter are 83\% and 88\%, respectively.}
    \label{fig:token}
    \vspace{-0.12in}
\end{figure}
In this section, we visualize the \textit{expert skipping} ratios of MoDES across modalities and layers to interpret the effectiveness of our approach. As shown in Fig.~\ref{fig:token}, our method skips substantially more experts in deeper layers than in shallower layers, which is consistent with the key insight discussed in Sec.~\ref{sec:motivation_global}. In addition, it skips far more experts for vision tokens than for text tokens, indicating greater redundancy among experts for vision tokens. We corroborate this observation with experiments in the Appendix. These results suggest that a uniform, modality-agnostic skipping schedule is inappropriate. This finding also reinforces the second insight in Sec.~\ref{sec:motivation_modality} and helps explain how our method preserves the model’s strong performance.

\section{Conclusions}
\label{sec:conclusion}
In this work, we proposed \textit{MoDES}, a novel framework for \textit{expert skipping} in MoE multimodal large language models (MLLMs). First, we identified two key insights: The imbalance of expert contributions across layers and the distinct behaviors between modalities in FFNs. Based on these findings, we introduced a \textit{globally-modulated local gating (GMLG)} mechanism and a \textit{dual-modality thresholding (DMT)} method, which allow the model to adaptively skip experts based on layer-specific importance and modality-specific characteristics. Additionally, we developed an efficient \textit{frontier search} algorithm, which greatly improves search efficiency for threshold optimization. Extensive experiments on large-scale multimodal benchmarks demonstrate that MoDES provides significant computational savings without sacrificing performance.

\section*{Acknowledgement}
This work was supported by the National Natural Science Foundation of China (Nos. 62476018), and the Postdoctoral Fellowship Program of CPSF (No. BX20250487). This work was also supported by the Hong Kong Research Grants Council under the Areas of Excellence scheme grant AoE/E-601/22-R and NSFC/RGC Collaborative Research Scheme grant CRS\_HKUST603/22.

{
    \small
    \bibliographystyle{ieeenat_fullname}
    \bibliography{main}
}

\newpage
\appendix
\begin{center}{\bf \Large Appendix}\end{center}\vspace{-2mm}
\renewcommand{\thetable}{\Roman{table}}
\renewcommand{\thefigure}{\Roman{figure}}
\renewcommand{\theequation}{\Roman{equation}}
\setcounter{table}{0}
\setcounter{figure}{0}
\setcounter{equation}{0}

\Crefname{appendix}{Appendix}{Appendixes}

This document supplements the main paper as follows:
\begin{itemize}[leftmargin=*]
    \item Sec.~\ref{sec:proof} provides detailed proofs for the proposed \textit{frontier search};
    \item Sec.~\ref{sec:more_steups} details additional experimental setups;
    \item Sec.~\ref{sec:more_compare} provides additional comparisons with baselines across different \textit{expert skipping} ratios and MLLMs;
    \item Sec.~\ref{sec:reasoning_vis} presents visual question answering examples across methods;
    \item Sec.~\ref{sec:ablation_n} reports ablations on the number of grid points in \textit{frontier search};
     \item Sec.~\ref{sec:ablation_d} shows ablations on the number of samples used for calibration and search;
    \item Sec.~\ref{sec:redundancy} analyzes expert redundancy \emph{w.r.t.} tokens across modalities.
\end{itemize}

\begin{table*}[!ht]\setlength{\tabcolsep}{4pt}
 \renewcommand{\arraystretch}{1}
  \centering
  \caption{Performance comparisons for Qwen3-VL-MoE-30B-A3B-Instruct~\cite{qwen3_vl_moe_doc} across various expert skipping ratios.}
  \vspace{-0.1in}
  \resizebox{\linewidth}{!}{
  \begin{tabu}[t!]{l|ccccccccccccc|c}
\toprule
\multirow{2}{*}{\textbf{Method}} & \multicolumn{8}{c}{\textbf{Image Understanding}} & \multicolumn{5}{c}{\textbf{Video Understanding}} & \multirow{2}{*}{\makecell{\textbf{Avg.}\\ \textbf{(\%)}}}\\
\cmidrule(lr){2-9}\cmidrule(lr){10-14} 
& \rotatebox{0}{TextVQA} & \rotatebox{0}{ChartQA} & \rotatebox{0}{MMStar} & \rotatebox{0}{MMBench} & \rotatebox{0}{MMVet} & \rotatebox{0}{MME} & \rotatebox{0}{RealWorldQA} & \rotatebox{0}{COCO} & \rotatebox{0}{MVBench} & \rotatebox{0}{EgoSchema} & \rotatebox{0}{VMME} & \rotatebox{0}{LVB} & \rotatebox{0}{VMMMU} &\\
\midrule
$k=8$ (\textit{Default}) & 83.41 & 85.08 & 59.67 & 86.60 & 69.68 & 2500 & 66.80 & 80.37 & 64.67 & 62.45 & 54.89 & 55.42 & 47.11 & 100.00 \\
\midrule

\multicolumn{15}{c}{\cellcolor[gray]{0.92} \textit{Skip $63\%$ Experts} ($\rho=0.60$)} \\
$k=3$ & 80.81 & 78.12 & 66.74 & 83.33 & 68.39 & 2326 & 45.88 & 71.70 & 62.02 & 57.96 & 53.48 & 54.60 & \textbf{50.44} & 95.20 \\
NAEE~\cite{lu2024not} & 81.20 & 79.41 & 55.39 & 84.18 & 68.61 & 2348 & 59.67 & \textbf{78.09} & 61.31 & 58.32 & 51.08 & \underline{55.12} & 48.32 & 95.61\\
MC-MoE~\cite{huang2024mixture} & \textbf{82.51} & 79.37 & 56.48 & \underline{86.12} & \underline{69.37} & \underline{2438} & \underline{62.01} & \underline{76.82} & 62.61 & 58.73 & \underline{54.22} & 54.13 & 48.54 & \underline{97.09} \\
DiEP~\cite{bai2025diep} & \underline{82.04} & \underline{80.23} & \underline{57.26} & 85.07 & 68.42 & 2405 & 60.31 & 75.41 & \underline{63.15} & \underline{59.46} & 53.41 & 55.08 & 48.76 & 96.80 \\
\rowcolor{mycolor!30} MoDES (\textit{Ours}) & 81.82 & \textbf{82.48} & \textbf{58.61} & \textbf{86.17} & \textbf{69.95} & \textbf{2493} & \textbf{63.92} & 76.55 & \textbf{64.42} & \textbf{62.39} & \textbf{55.15} & \textbf{55.50} & \underline{49.89} & \textbf{99.22} \\
\midrule

\multicolumn{15}{c}{\cellcolor[gray]{0.92} \textit{Skip $75\%$ Experts} ($\rho=0.73$)} \\
$k=2$ & 77.54 & 69.60 & 62.38 & 80.50 & 61.33 & 2060 & 55.56 & \textbf{82.77} & 60.70 & 53.79 & 50.67 & 54.08 & 46.00 & 92.03\\
NAEE~\cite{lu2024not} & 78.42 & 77.28 & 54.64 & 81.34 & 65.58 & 2208 & 61.75 & \underline{77.31} & 60.98 & 55.24 & 48.87 & \underline{54.87} & 47.12 & 93.25\\
MC-MoE~\cite{huang2024mixture} & \underline{80.13} & 78.41 & \underline{57.02} & \underline{85.32} & \underline{67.22} & \underline{2286} & \underline{61.83} & 74.49 & 61.65 & 57.13 & \underline{52.64} & 54.03 & 47.49 & 95.11 \\
DiEP~\cite{bai2025diep} & 79.64 & \underline{78.52} & 56.48 & 84.91 & 67.13 & 2243 & 60.94 & 75.53 & \underline{62.78} & \underline{57.86} & 52.38 & 54.62 & \underline{48.16} & \underline{95.21} \\
\rowcolor{mycolor!30} MoDES (\textit{Ours}) & \textbf{81.65} & \textbf{82.44} & \textbf{58.78} & \textbf{86.25} & \textbf{67.61} & \textbf{2469} & \textbf{64.71} & 75.73 & \textbf{64.45} & \textbf{62.53} & \textbf{54.81} & \textbf{55.57} & \textbf{51.22} & \textbf{99.11} \\
\midrule

\multicolumn{15}{c}{\cellcolor[gray]{0.92} \textit{Skip $88\%$ Experts} ($\rho=0.85$)} \\
$k=1$ & 60.71 & 52.16 & 31.63 & 54.90 & 28.07 & 1590 & 52.42 & 45.64 & 41.51 & 32.52 & 39.78 & 42.41 & 12.51 & 60.11\\
NAEE~\cite{lu2024not} & 72.41& 65.83& 48.88& 73.62& 54.52& 1984& 58.62& 60.37& 50.24& 49.77& 44.48& 45.59& 35.57 & 80.60 \\
MC-MoE~\cite{huang2024mixture} & \underline{74.87} & \underline{71.43}& 50.74& \underline{75.42}& \underline{61.35}& \underline{2168}& 60.41& \underline{68.15}& 56.60& 51.84& \underline{52.51}& \underline{47.22}& \underline{37.41} & \underline{86.66} \\
DiEP~\cite{bai2025diep} & 73.46& 70.51& \underline{53.28}& 73.21& 58.64& 2074& \underline{63.41}& 62.89& \underline{57.21}& \underline{53.61}& 50.78& 46.13& 34.79 & 85.30 \\
\rowcolor{mycolor!30} MoDES (\textit{Ours}) & \textbf{80.97} & \textbf{78.84} & \textbf{58.18} & \textbf{85.57} & \textbf{67.75} & \textbf{2403} & \textbf{64.58} & \textbf{74.66} & \textbf{62.98}& \textbf{62.04} & \textbf{55.26} & \textbf{55.50} & \textbf{46.56} & \textbf{97.33}\\

\bottomrule
\end{tabu}
}
    \label{tab:compare-qwen}
    \vspace{-0.1in}
\end{table*}
\begin{table*}[!ht]\setlength{\tabcolsep}{4pt}
 \renewcommand{\arraystretch}{1.}
  \centering
  \caption{Performance comparisons for InternVL-3.5-30B-A3B-HF~\cite{wang2025internvl3} across various expert skipping ratios.}
  \vspace{-0.1in}
  \resizebox{\linewidth}{!}{
  \begin{tabu}[t!]{l|ccccccccccccc|c}
\toprule
\multirow{2}{*}{\textbf{Method}} & \multicolumn{8}{c}{\textbf{Image Understanding}} & \multicolumn{5}{c}{\textbf{Video Understanding}} & \multirow{2}{*}{\makecell{\textbf{Avg.}\\ \textbf{(\%)}}}\\
\cmidrule(lr){2-9}\cmidrule(lr){10-14} 
& \rotatebox{0}{TextVQA} & \rotatebox{0}{ChartQA} & \rotatebox{0}{MMStar} & \rotatebox{0}{MMBench} & \rotatebox{0}{MMVet} & \rotatebox{0}{MME} & \rotatebox{0}{RealWorldQA} & \rotatebox{0}{COCO} & \rotatebox{0}{MVBench} & \rotatebox{0}{EgoSchema} & \rotatebox{0}{VMME} & \rotatebox{0}{LVB} & \rotatebox{0}{VMMMU} &\\
\midrule

$k=8$ (\textit{Default}) & 85.76 & 84.08 & 62.49 & 83.81 & 69.93 & 2312 & 64.77 & 69.30 & 68.92 & 60.49 & 58.07 & 57.64 & 45.11 & 100.00\\
\midrule

\multicolumn{15}{c}{\cellcolor[gray]{0.92} \textit{Skip $63\%$ Experts} ($\rho=0.60$)} \\
$k=3$ & 82.16 & 81.38 & 60.30 & 77.94 & \textbf{68.67} & 1964 & 61.34 & 65.47 & 65.34 & 58.83 & 55.62 & 55.81 & 42.07 & 94.79\\
NAEE~\cite{lu2024not} & 82.98& 83.02& 61.18& 79.65& 67.57& 2054& 61.47& 66.05& 66.73& 58.46& 56.34& 55.74& 42.81 & 95.86 \\
MC-MoE~\cite{huang2024mixture} & \textbf{84.36}& \textbf{83.22}& 61.45& \underline{80.89}& \textbf{68.67}& \underline{2192}& 62.13& 66.87& \underline{67.38}& \underline{59.03}& \underline{56.79}& 56.02& \underline{43.45} & \underline{97.25}\\
DiEP~\cite{bai2025diep} & 83.68& 82.79& \underline{61.82}& 80.22& 68.13& 2084& \underline{62.56}& \underline{67.17}& 66.82& 58.74& 56.25& \textbf{57.84}& 43.16 & 96.82\\
\rowcolor{mycolor!30} MoDES (\textit{Ours}) & \underline{84.27} & \underline{83.15} & \textbf{62.06} & \textbf{81.46} & \underline{68.41} & \textbf{2289} & \textbf{63.10} & \textbf{68.22} & \textbf{68.64} & \textbf{60.15} & \textbf{57.76} & \underline{56.12} & \textbf{43.84} & \textbf{98.42}\\

\midrule
\multicolumn{15}{c}{\cellcolor[gray]{0.92} \textit{Skip $75\%$ Experts} ($\rho=0.73$)} \\
$k=2$ & 64.51 & 64.25 & 46.69 & 71.56 & 56.42 & 1821 & 57.29 & 58.28 & 61.42 & 53.25 & 51.06 & 48.87 & 38.63 & 83.02\\
NAEE~\cite{lu2024not} & 75.37& 76.18& \underline{58.82}& 74.53& 61.38& 1968& 59.47& 63.31& 64.46& 54.83& \underline{55.45}& 52.79& 41.08 & 90.76\\
MC-MoE~\cite{huang2024mixture} & \underline{77.41}& 78.24& 57.65& 75.58& \underline{66.41}& \underline{2037}& \underline{60.28}& \underline{64.24}& \underline{65.18}& \underline{56.14}& 53.65& 53.08& \underline{41.74} & \underline{92.30} \\
DiEP~\cite{bai2025diep} & 76.84& \underline{79.12}& 58.42& \underline{76.14}& 65.27& 2021& 58.74& 63.10& 64.89& 55.83& 54.12& \underline{54.22}& 40.23& 91.80\\
\rowcolor{mycolor!30} MoDES (\textit{Ours}) & \textbf{82.13} & \textbf{82.54} & \textbf{61.46} & \textbf{81.88} & \textbf{67.92} & \textbf{2258 }& \textbf{62.48} & \textbf{67.89} & \textbf{68.83} & \textbf{60.32} & \textbf{57.54} & \textbf{55.85} & \textbf{44.16} & \textbf{97.90} \\

\midrule
\multicolumn{15}{c}{\cellcolor[gray]{0.92} \textit{Skip $88\%$ Experts} ($\rho=0.85$)} \\
$k=1$ & 58.49 & 46.24 & 42.27 & 51.74 & 35.05 & 1683 & 51.44 & 26.01 & 31.99 & 34.47 & 35.26 & 37.40 & 24.27 & 59.63\\
NAEE~\cite{lu2024not} & 66.24& 68.32& 50.14& 64.37& 49.52& 1802& 55.23& 50.64& 54.78& 50.25& 48.69& 47.42& 37.27 & 78.88\\
MC-MoE~\cite{huang2024mixture} & \underline{70.41}& \underline{73.49}& 56.14& \underline{64.38} & \underline{65.41}& \underline{1972}& \underline{57.49}& \underline{60.12}& \underline{58.97}& \underline{52.31}& \underline{49.72}& \underline{48.31}& \underline{40.06} & \underline{86.20} \\
DiEP~\cite{bai2025diep} & 69.37& 71.84& \underline{57.21}& 63.19& 65.32& 1838& 56.38& 55.78& 56.26& 51.48& 48.94& 47.26& 38.18 & 83.26\\
\rowcolor{mycolor!30} MoDES (\textit{Ours}) & \textbf{80.58} & \textbf{82.00} & \textbf{61.20} & \textbf{81.67} & \textbf{67.80} & \textbf{2222} & \textbf{61.73} & \textbf{65.16} & \textbf{68.65} & \textbf{60.79} & \textbf{57.63} & \textbf{54.49} & \textbf{44.33} & \textbf{97.03}\\

\bottomrule
\end{tabu}
}
    \label{tab:compare-intern}
    \vspace{-0.1in}
\end{table*}

\section{Proofs}\label{sec:proof}
In this section, we first provide complete proofs of the correctness and time complexity for our frontier search (Prop.~\ref {prop:correct-time}). We then prove that the optimal thresholds lie on the frontier (Prop.~\ref{prop:opt-frontier}).

\begin{lemma}[Monotone feasibility in $p$]\label{lem:mono-p}
\textit{For fixed $q$, define
\begin{equation}
\Phi_q(p) := \bigl[g(\tau^{(q)},\tau^{(p)}) \ge \rho\bigr].
\end{equation}
If $g$ is non-decreasing in its second argument, then $\Phi_q(p)$ is monotone in $p$.
Hence, if a feasible $p$ exists, the smallest feasible index
\begin{equation}
p_{(q)} := \min\{\,p:\Phi_q(p)\,\}
\end{equation}
is well-defined.}
\end{lemma}

\begin{proof}
If $p_1 \le p_2$ and $\Phi_q(p_1)$ holds, then by monotonicity of $g$ in its second argument,
\begin{equation}
g(\tau^{(q)},\tau^{(p_2)}) \ge g(\tau^{(q)},\tau^{(p_1)}) \ge \rho,
\end{equation}
so $\Phi_q(p_2)$ holds. Therefore, the feasible set is a suffix in $p$, and the minimum exists when the set is non-empty.
\end{proof}

\begin{lemma}[Monotone shift in $q$]\label{lem:mono-q}
\textit{Assume $g$ is non-decreasing in its first argument. If $q' \le q$ and both $p_{(q')}$ and $p_{(q)}$ exist, then
\begin{equation}
p_{(q)} \le p_{(q')}.
\end{equation}}
\end{lemma}

\begin{proof}
For any fixed $p$ and $q' \le q$,
\begin{equation}
g(\tau^{(q)},\tau^{(p)}) \ge g(\tau^{(q')},\tau^{(p)}).
\end{equation}
Hence
\begin{equation}
\{\,p:\Phi_q(p)\,\} \supseteq \{\,p:\Phi_{q'}(p)\,\}.
\end{equation}
Taking minima over these sets gives $p_{(q)} \le p_{(q')}$.
\end{proof}

\begin{lemma}[Loop invariant]\label{lem:invariant}
\textit{Let $p$ be the pointer value at the start of the $q$-th outer iteration in Alg.~\ref{alg:frontier_search}.
If $p_{(q)}$ exists, then
\begin{equation}
p \ge p_{(q)} - 1 .
\end{equation}
Moreover, after the inner loop for this $q$, the algorithm stops at $p = p_{(q)} - 1$ and records $p_{(q)} = p+1$.}
\end{lemma}

\begin{proof}
Base case ($q=1$): The algorithm sets $p \leftarrow D$, and $D \ge p_{(1)}$, so the claim holds.

Inductive step: Assume that the claim holds for $q$.
By Lem.~\ref{lem:mono-p}, $\Phi_q$ is monotone in $p$.
The inner loop decreases $p$ until $\neg \Phi_q(p)$ holds for the first time.
Thus, it stops at $p = p_{(q)} - 1$, and the code sets $p_{(q)} \leftarrow p+1$. For the next iteration, the carried pointer is $p \leftarrow p_{(q)} - 1$. By Lem.~\ref{lem:mono-q}, $p_{(q+1)} \le p_{(q)}$, hence
\begin{equation}
p = p_{(q)} - 1 \ge p_{(q+1)} - 1.
\end{equation}
Thus, the invariant holds for $q+1$.
\end{proof}

\begin{proposition}[Correctness and time]\label{prop:correct-time}
\textit{Assume $g$ is non-decreasing in each argument.
Then Lines 1-12 of Alg.~\ref{alg:frontier_search} compute the frontier $\{(q,p_{(q)})\}$.
If each evaluation of $(f,g)$ on $\mathcal{C}$ costs $\mathcal{O}(N)$ time, the total time is $\mathcal{O}(ND)$.}
\end{proposition}

\begin{proof}
By Lem.~\ref{lem:mono-p}, each feasible $p_{(q)}$ is well-defined.
By Lem.~\ref{lem:invariant}, at the $q$-th iteration the inner loop stops at $p=p_{(q)}-1$ and records $p_{(q)}=p+1$, which is the smallest feasible index.
If no feasible $p$ exists for some $q$, then $\Phi_q(D)$ is false and the guard $p_{(q)} \le D$ excludes this $q$, as desired.
Therefore, Lines 1-12 are correct.

For the time bound, by Lem.~\ref{lem:mono-q}, $p_{(q)}$ is non-increasing in $q$. Hence, across all outer iterations, the while-guard inspects $g$ at most $D$ times when $p$ is decremented and at most $D$ additional times when the guard fails immediately at the start of an iteration, so the total number of guard evaluations of $g$ is at most $2D$ (\emph{i.e.}, $\mathcal{O}(D)$). Moreover, for each recorded frontier element $(q,p_{(q)})$ (at most $D$ in total), we use a single forward pass that computes $f(\tau^{(q)},\tau^{(p_{(q)})})$. Each evaluation costs $\mathcal{O}(N)$. Therefore, the total time is $\mathcal{O}(ND)$.
\end{proof}

\noindent\textbf{Implementation note.} In practice, we compute $f$ and $g$ simultaneously and can record their values. This merges their costs and reduces constant factors, while the asymptotic bound remains $\mathcal{O}(ND)$.

\begin{lemma}[Frontier suffices]\label{lem:frontier-suffices}
\textit{Assume $f$ is non-decreasing in each argument and $\mathcal{F}=\{(q,p): g(\tau^{(q)},\tau^{(p)})\ge \rho\}\neq\emptyset$.
For any fixed feasible $q$, the pair $(q,p_{(q)})$ satisfies
\begin{equation}
f\bigl(\tau^{(q)},\tau^{(p_{(q)})}\bigr) \le f\bigl(\tau^{(q)},\tau^{(p)}\bigr)
\quad \text{for all } (q,p)\in\mathcal{F}.
\end{equation}}
\end{lemma}

\begin{proof}
By definition, $p \ge p_{(q)}$ for all feasible $(q,p)$.
Since $f$ is non-decreasing in its second argument,
\begin{equation}
f\bigl(\tau^{(q)},\tau^{(p_{(q)})}\bigr) \le f\bigl(\tau^{(q)},\tau^{(p)}\bigr).
\end{equation}
\end{proof}

\begin{proposition}[Optimality on the frontier]\label{prop:opt-frontier}
\textit{Under the assumptions of Lem.~\ref{lem:frontier-suffices}, any optimal solution of
\begin{equation}
\min_{(q,p)\in \{1,\dots,D\}^2} f(\tau^{(q)},\tau^{(p)})
\quad \text{s.t.}\quad g(\tau^{(q)},\tau^{(p)}) \ge \rho
\end{equation}
lies on the frontier $\{(q,p_{(q)})\}$.}
\end{proposition}

\begin{proof}
By Lem.~\ref{lem:frontier-suffices}, for each feasible $q$, the best feasible choice in $p$ is $p_{(q)}$.
Therefore, an optimal pair can be chosen from
\begin{equation}
\bigl\{\, (q,p_{(q)}):\ p_{(q)}\ \text{exists} \,\bigr\},
\end{equation}
which is exactly the frontier.
This is what Lines~13–14 minimize over, using the $f$-values already stored when each $(q,p_{(q)})$ was inserted into the frontier.
\end{proof}

\begin{figure*}[!ht]
    \vspace{-0.1in}
    \centering
    \includegraphics[width=\linewidth]{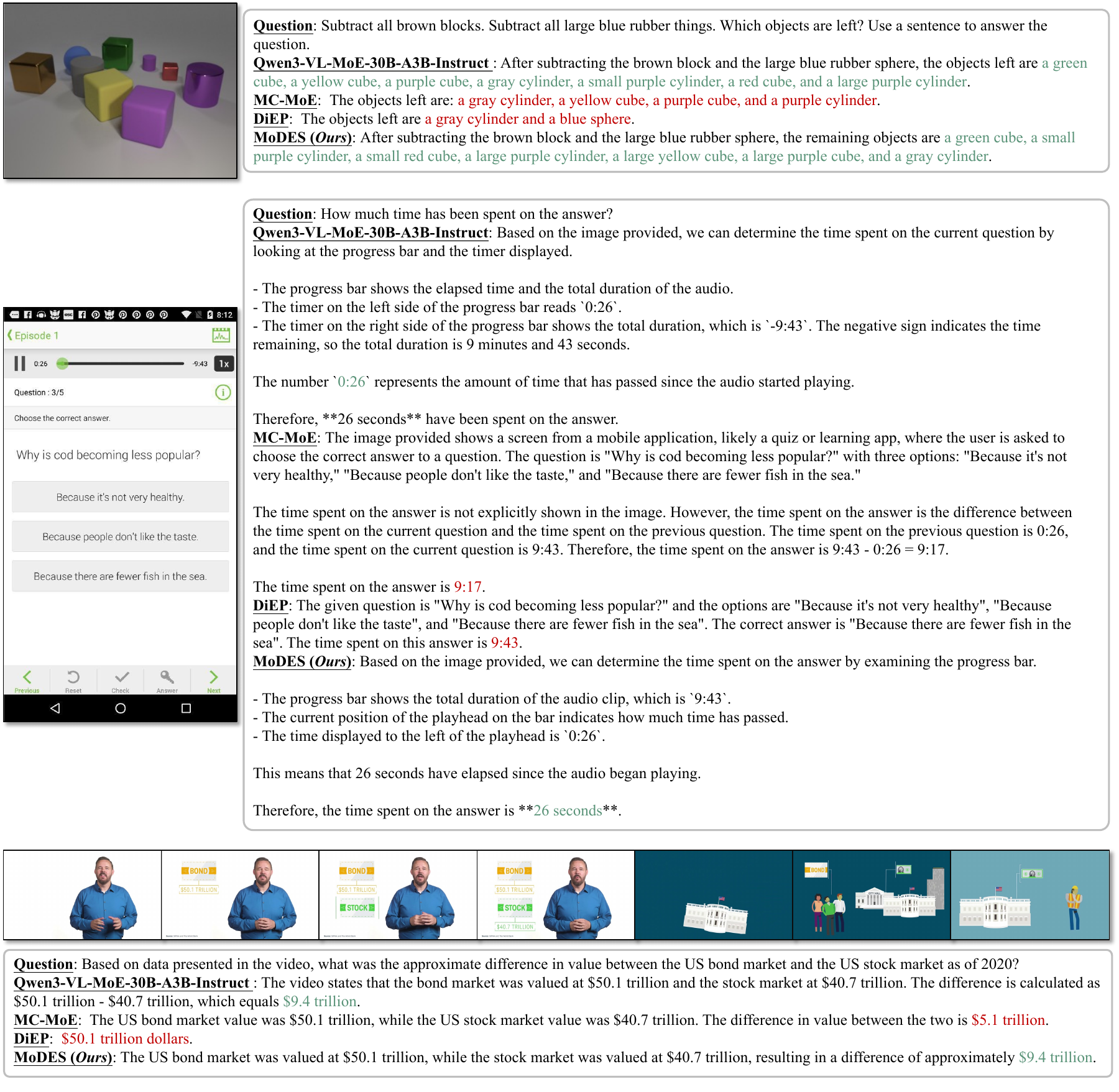}
    \vspace{-0.25in}
     \caption{Visual understanding examples from Qwen3-VL-MoE-A3B-Instruct~\cite{qwen3_vl_moe_doc}. We employ an $88\%$ \textit{expert skipping ratio} for \textit{all} methods, and color the text to show the \mygreen{correct} or the \myred{wrong} responses.}
    \label{fig:vis_exp_2}
    \vspace{-0.15in}
\end{figure*}
\begin{figure*}[!ht]
    \vspace{-0.1in}
    \centering
    \includegraphics[width=\linewidth]{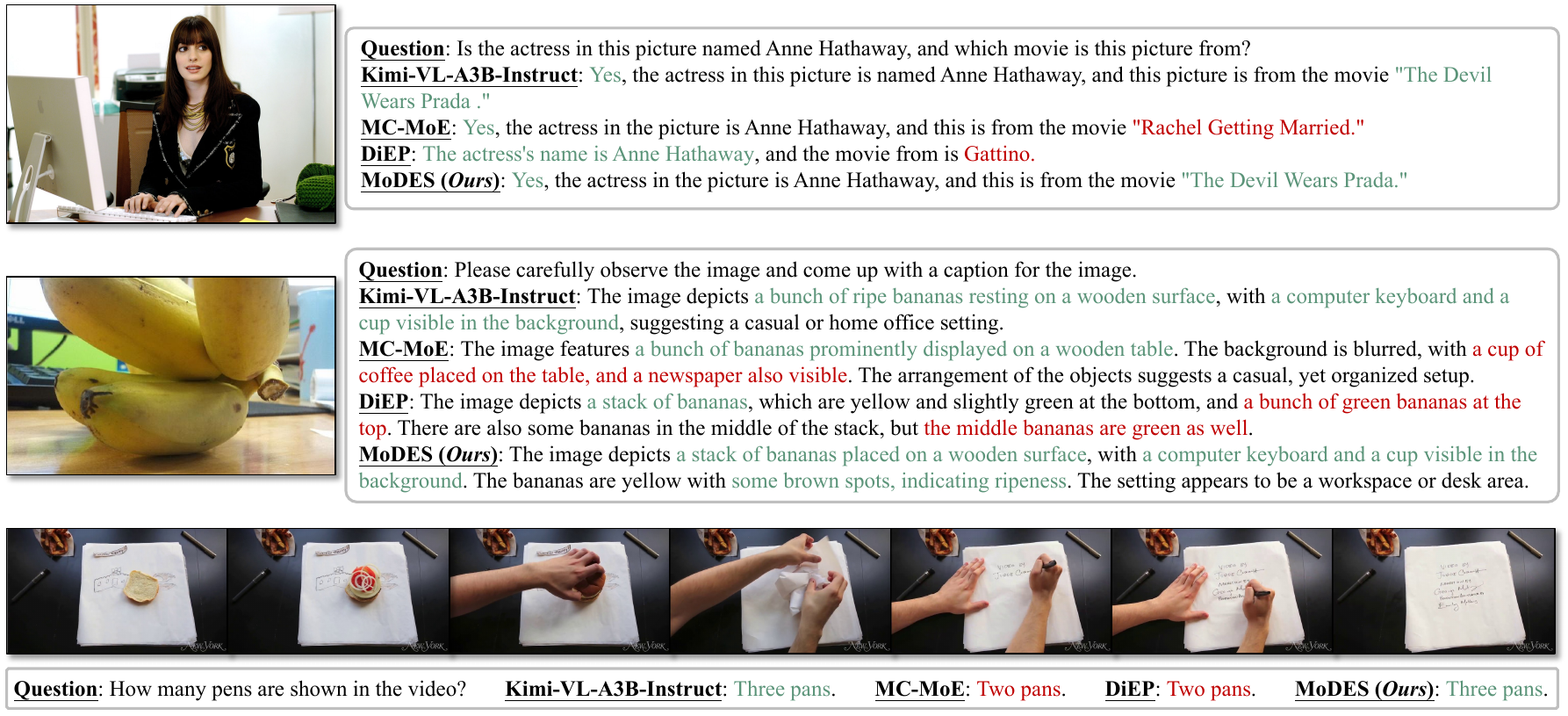}
    \vspace{-0.25in}
     \caption{Visual understanding examples from Kimi-VL-A3B-Instruct~\cite{team2025kimi}. We employ an $83\%$ \textit{expert skipping ratio} for \textit{all} methods.}
    \label{fig:vis_exp}
    \vspace{-0.15in}
\end{figure*}

\section{More Setups}\label{sec:more_steups}
\noindent\textbf{Baselines.} As noted in Sec.~\ref{sec:imple}, baselines such as NAEE~\cite{lu2024not}, MC-MoE~\cite{huang2024mixture}, and DiEP~\cite{bai2025diep} are not directly compatible with MoE MLLMs when ($k>2$). We therefore describe more about our adaptations here. For the hyperparameter $\beta^{(l)}$, we perform a genetic search under a given skipping ratio on the same dataset as our method. All remaining settings follow the original papers.

\noindent\textbf{Implementation.} In practice, we normalize $\alpha^{(l)}$ across layers as $\widetilde{\alpha^{(l)}}=\frac{\alpha^{(l)}}{\sum^L_{l'=1}\alpha^{(l')}}$. During inference, we compute $s^{(l)}_i=\widetilde{\alpha^{(l)}}\cdot \pi^{(l)}_i$ for a given token $\mathbf{x}^{(l)}$. Since $0<\pi^{(l)}_i<1$ ($i\in\mathcal{S}^{(l)}$), $s^{(l)}_i\in(0, 1)$. Thus, we choose $D=100$ grids in $(0, 1)$ as $\mathcal{B}$ to search for optimal thresholds. In detail, we apply a rectified \texttt{sigmoid} function to $100$ grids falling into $[0, 1]$ with equal intervals. 

For inference speed measurement, we write efficient CUDA kernels for MoE layers. 
First, we implement our dual-modality thresholding method inside the router kernel, so it introduces no extra kernel launches or separate decision pass. After computing router logits and top-$k$, we apply a branch-free masked comparison with the modality-specific threshold and directly edit the top-$k$ outputs: skipped routes are assigned an invalid sentinel expert id (\emph{e.g.}, $M{+}1$). During MoE dispatch/gather, sentinel entries are filtered out and thus never scheduled/executed, reducing both expert compute and expert loading. The added cost is only a few element-wise operations on the small top-$k$ list, so warp divergence/overhead is minimal and does not negate the observed wall-clock speedups. Moreover, to efficiently execute the computations for the activated experts, we employ a Grouped General Matrix Multiplication (Group GEMM) approach. Group GEMM enables the concurrent execution of all required matrix multiplications within a single, unified kernel launch. Each expert's computation is treated as an independent sub-task within the group. The performance of this kernel is highly dependent on the workload distribution. Therefore, to achieve maximum efficiency, we perform an offline profiling step where we conduct a grid search to identify the optimal kernel tile sizes for various representative activation patterns. This ensures high computational throughput across the diverse and dynamic workloads characteristic of MoDES computation.

All performance experiments are conducted on $8\times$H200 GPUs, and efficiency experiments are performed on a single H200 GPU.

\section{More Comparison with Baselines}\label{sec:more_compare}
We provide additional results for the Qwen3-VL-MoE-30B-A3B-Instruct~\cite{qwen3_vl_moe_doc} and  InternVL-3.5-30B-A3B-HF~\cite{wang2025internvl3} in Tabs.~\ref{tab:compare-qwen} and \ref{tab:compare-intern}, respectively. The observations from these
results align with the phenomena identified in Kimi-VL-A3B-Instruct~\cite{team2025kimi}. Across different \textit{expert skipping} ratios, our method consistently outperforms the baselines, with especially large gains at high skipping levels ($\geq$75\%).

\section{Visual Understanding Visualization}\label{sec:reasoning_vis}
In this section, we present a case study comparing our proposed MoDES with previous SOTA methods~\cite{bai2025diep,huang2024mixture} for LLMs. As shown in Figs.~\ref{fig:vis_exp_2} and \ref{fig:vis_exp}, MoDES consistently generates text that far outperforms the baselines.

\section{Ablation for \texorpdfstring{$N$}{N}}\label{sec:ablation_n}

\begin{table}[!ht]\setlength{\tabcolsep}{5.5pt}
\vspace{-0.1in}
 \renewcommand{\arraystretch}{1.}
  \centering
  \caption{Ablation results for $N$.} 
  \vspace{-0.1in}
  \resizebox{0.88\linewidth}{!}{
  \begin{tabu}[t!]{l|cccccc}
\toprule
$N$ & \rotatebox{0}{ChartQA} & \rotatebox{0}{MME} & MMBench  & \rotatebox{0}{LVB} & \rotatebox{0}{VMMMU} \\
\midrule
\multicolumn{7}{c}{\cellcolor[gray]{0.92} \textit{Skip $67\%$ Experts} ($\rho=0.65$)} \\
\midrule
2048 & \textbf{88.32} & 2201 & \textbf{82.79} & \textbf{62.92} & \textbf{48.89} \\
\rowcolor{mycolor!30}1024 (\textit{Ours}) & \underline{88.24} & \textbf{2204} & \underline{82.73} & \underline{62.90} & \underline{48.78} \\
512 & 87.44 & 2122 & 81.27 & 61.95 & 47.68 \\
256 & 85.56 & 2085 & 79.68 & 60.63 & 45.11\\

\midrule
\multicolumn{7}{c}{\cellcolor[gray]{0.92} \textit{Skip $83\%$ Experts} ($\rho=0.80$)} \\
\midrule

2048 & \textbf{84.84} & \textbf{2186} & \textbf{81.45} & \textbf{62.63} & \underline{46.67}\\
\rowcolor{mycolor!30}1024 (\textit{Ours}) & \underline{84.20} & \underline{2162} & \underline{81.44} & \underline{62.60} & \textbf{47.11} \\
512 & 84.12 & 2118 & 80.27 & 61.88 & 46.85 \\
256 & 83.35 & 2016 & 77.48 & 59.84 & 43.69 \\

\bottomrule
\end{tabu}
}
    \label{tab:ablation_n}
    \vspace{-0.1in}
\end{table}
We apply MoDES to Kimi-VL-A3B-Instruct~\cite{team2025kimi} using different numbers of data samples from
GQA~\cite{hudsom2019gqa} and show the results in Tab.~\ref{tab:ablation_n}. The results indicate a clear trend: With more calibration samples, models using \textit{expert skipping} perform better. Yet the accuracy gains become smaller as the sample count grows. Moreover, doubling the samples increases both calibration and search time by $\sim$2$\times$. To balance accuracy and cost, we use 1024 samples in this paper. This choice provides most of the achievable gains while keeping computation reasonable (Sec.~\ref{sec:efficiency}).

\section{Ablation for \texorpdfstring{$D$}{D}}\label{sec:ablation_d}
\begin{table}[!ht]\setlength{\tabcolsep}{5.5pt}
\vspace{-0.1in}
 \renewcommand{\arraystretch}{1.}
  \centering
  \caption{Ablation results for $D$.} 
  \vspace{-0.1in}
  \resizebox{0.88\linewidth}{!}{
  \begin{tabu}[t!]{l|cccccc}
\toprule
$D$ & \rotatebox{0}{ChartQA} & \rotatebox{0}{MME} & MMBench  & \rotatebox{0}{LVB} & \rotatebox{0}{VMMMU} \\
\midrule
\multicolumn{7}{c}{\cellcolor[gray]{0.92} \textit{Skip $67\%$ Experts} ($\rho=0.65$)} \\
\midrule
200 & \underline{88.16} & \textbf{2219} & \textbf{82.78} & \textbf{62.94} & \underline{48.76} \\
\rowcolor{mycolor!30}100 (\textit{Ours}) & \textbf{88.24} & \underline{2204} & \underline{82.73} & \underline{62.90} & \textbf{48.78} \\
50 & 87.85 & 2178 & 81.76 & 62.21 & 47.89\\

\midrule
\multicolumn{7}{c}{\cellcolor[gray]{0.92} \textit{Skip $83\%$ Experts} ($\rho=0.80$)} \\
\midrule

200 & \textbf{84.78} & \textbf{2178} & \textbf{81.61} & 62.59 & 47.00\\
\rowcolor{mycolor!30}100 (\textit{Ours}) & \underline{84.20} & \underline{2162} & \underline{81.44} & \textbf{62.60} & \underline{47.11} \\
50 & 83.96 & 2143 & 80.68 & 62.47 & \textbf{47.15}\\

\bottomrule
\end{tabu}
}
    \label{tab:ablation_d}
    \vspace{-0.1in}
\end{table}
We ablate the number of grid points $D$ in the search space $\mathcal{B}$. As shown in Tab.~\ref{tab:ablation_d}, larger $D$ brings diminishing accuracy gains, so using a very fine grid (\emph{e.g.}, $D>100$) is unnecessary. The time cost also grows roughly linearly with $D$. Based on this trade-off, we set $D=100$ in this work.

\section{Expert Redundancy across Modalities}\label{sec:redundancy}
\begin{figure}[!ht]
\vspace{-0.1in}
   \centering
    \setlength{\abovecaptionskip}{0.2cm}
   \begin{minipage}[b]{\linewidth}
       \begin{minipage}[b]{0.326\linewidth}
            \centering
            \begin{subfigure}[tp!]{\textwidth}
            \centering
            \subcaption{ChartQA~\cite{masry2022chartqa}}
            \includegraphics[width=\linewidth]{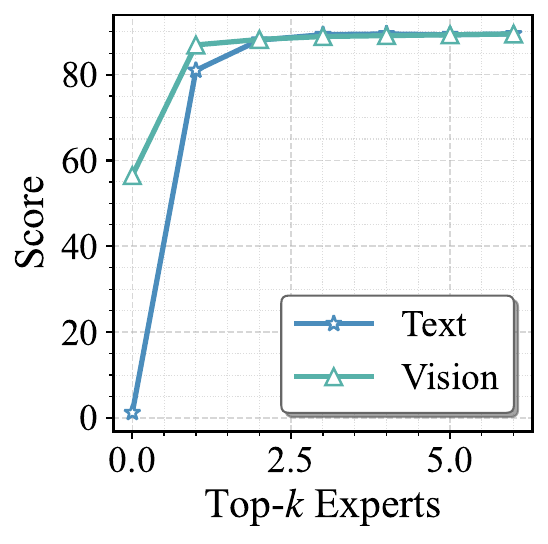}
            \end{subfigure}
       \end{minipage}
       \begin{minipage}[b]{0.326\linewidth}
            \centering
            \begin{subfigure}[tp!]{\linewidth}
            \centering
            \subcaption{MME~\cite{fu2023mme}}
            \includegraphics[width=\linewidth]{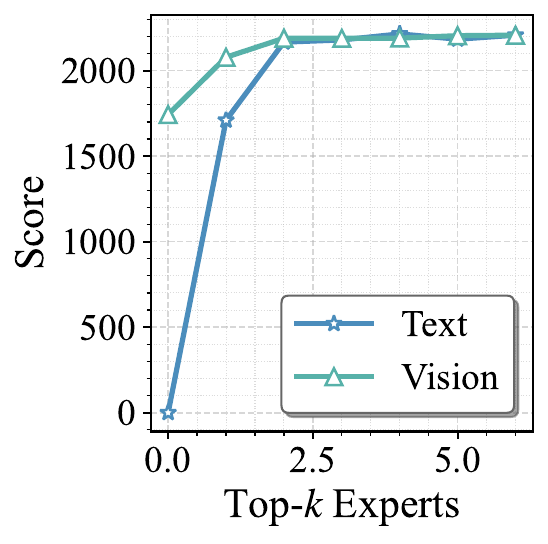}
            \end{subfigure}
       \end{minipage}
       \begin{minipage}[b]{0.326\linewidth}
            \centering
            \begin{subfigure}[tp!]{\linewidth}
            \centering
            \subcaption{VideoMMMU~\cite{hu2025videommmuevaluatingknowledgeacquisition}}
            \includegraphics[width=\linewidth]{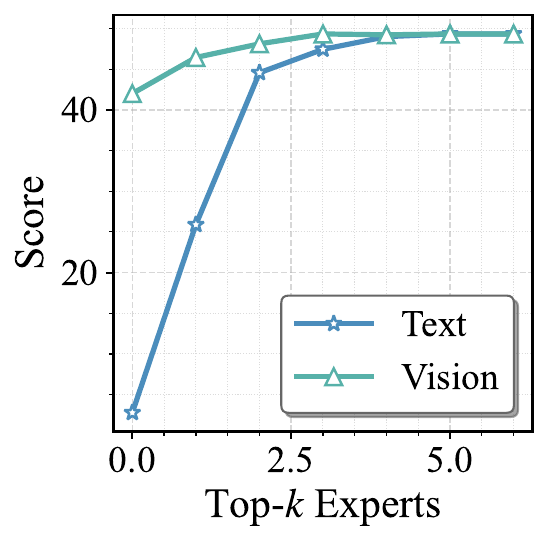}
            \end{subfigure}
       \end{minipage}
   \end{minipage}
   \vspace{-0.1in}
     \caption{Task performance across various numbers of top-$k$ routed experts applied to tokens of different modalities for Kimi-VL-A3B-Instruct~\cite{team2025kimi}.}
    \label{fig:expert_redundancy}
    \vspace{-0.15in}
\end{figure}

In this section, we analyze expert redundancy across modalities. As shown in Fig.~\ref{fig:expert_redundancy}, reducing $k$ for vision tokens causes task performance to drop more slowly than for text tokens. This indicates greater redundancy among experts for vision tokens, allowing more aggressive skipping than for text tokens. It also motivates modality-aware strategies for \textit{expert skipping}.

\end{document}